\documentclass[submission,copyright,creativecommons]{eptcs}
\providecommand{\event}{TARK 2025} % Name of the event you are submitting to
\usepackage{relsize}
\usepackage{doi}
  \usepackage{lscape} 
\usepackage{tikz}
\makeatletter
\def\and{%
  \end{tabular}%
  \hskip 0.27em \@plus.17fil\relax
  \begin{tabular}[t]{c}}
\makeatother
\tikzstyle{bag} = [align=center]
\usepackage{rotating}
\usepackage{amsmath}
\usepackage{amsfonts}
\usepackage{amssymb}
\usepackage{amsthm}
\usetikzlibrary{shapes}
\usepackage{graphicx}
\usepackage{stackengine}
\usepackage{xcolor}
\usepackage{thmtools}
\usepackage{hyperref}
\usepackage{hhline}
\usepackage{multirow}
\usepackage{comment}
\declaretheoremstyle[%
  spaceabove=-4pt,%
  spacebelow=6pt,%
  headfont=\normalfont\itshape,%
  postheadspace=1em,%
  qed=\qedsymbol%
]{mystyle}

\newtheorem{theorem}{Theorem}
\newtheorem{definition}{Definition}
\newtheorem{proposition}{Proposition}

\newtheorem{lemma}{Lemma}
\newtheorem{corollary}{Corollary}

\newtheorem{fact}{Fact}

\newcommand{\Nat}[0]{\mathbb{N}}

\usepackage{iftex}

\ifpdf
  \usepackage[T1]{fontenc}        % Recommended with pdflatex
\else
  \usepackage{breakurl}           % Not needed if you use pdflatex only.
\fi

\title{Who is Afraid of Minimal Revision?}
\author{Edoardo Baccini 
\institute{University of Groningen}
\email{e.baccini@rug.nl}
\and
Zo\'{e} Christoff
\institute{University of Groningen}
\email{z.l.christoff@rug.nl}
\and
Nina Gierasimczuk
\institute{Technical University of Denmark}
\email{
nigi@dtu.dk}
\and
Rineke Verbrugge
\institute{University of Groningen}
\email{l.c.verbrugge@rug.nl}
}
\def\titlerunning{Who is Afraid of Minimal Revision?}
\def\authorrunning{E. Baccini, Z. Christoff, N. Gierasimczuk, R. Verbrugge}
\begin{document}
\maketitle

\begin{abstract}
The principle of minimal change in belief revision theory requires that, when accepting new information, one keeps one's belief state as close to the initial belief state as possible. This is precisely what the method known as minimal revision does. However, unlike less conservative belief revision methods, minimal revision falls short in learning power: It cannot learn everything that can be learned by other learning methods. We begin by showing that, despite this limitation, minimal revision is still a successful learning method in a wide range of situations. Firstly, it can learn any problem that is finitely identifiable. Secondly, it can learn with positive and negative data, as long as one considers finitely many possibilities. We then  characterize the prior plausibility assignments (over finitely many possibilities) that enable one to learn via minimal revision, and do the same for conditioning and lexicographic upgrade. Finally, we show that not all of our results still hold when learning from possibly erroneous information.
\end{abstract}

\section{Introduction}

Many human activities aim to find the truth. 
% This is for instance the goal of scientists striving to discover theories about the universe, of jurors deliberating to reach a verdict, and of people consuming information to revise and form beliefs about the world. 
Consider, for instance, a scientist deciding which of two possible theories, $T_1$ or $T_2$, is the correct one.
At the start of their inquiry, the scientist might have some \textit{prior plausibility ordering} that reflects their preference among the two. Over time, new experimental results and observations become available to the scientist, who revises their assessment accordingly using some \textit{revision method}. How should the scientist revise their prior assessment? Some inspiration about the structure of that dynamics can be drawn from belief revision theory. In particular, one natural way to revise one's beliefs is to accept 
the newly learned information but change the overall belief state as little as possible. This \textit{minimal revision} method (also known as ``conservative upgrade'') was first introduced in \cite{Boutilier:1992}, where it was shown that the method is consistent with the basic AGM principles of (iterated) belief revision \cite{AGM}, that it naturally accommodates the Ramsey test for conditionals \cite{Stalnaker1968-STAATO-5}, and that, while continually updating the propositional beliefs, it preserves as many as possible of the extended, conditional beliefs. 

Despite its natural appeal, minimal revision falls short in its \textit{learning power}, i.e., in its ability to successfully lead an agent to identify the true state of the world among a set of possibilities.
% since not all problems that can be learned via some methods can be learned via the method of minimal revision. 
In particular, using tools from formal learning theory, \cite{kelly1999iterated} and \cite{NinaPhD,BGStruth-tracking,baltag2019truth} show that, unlike the more radical belief revision methods of \textit{conditioning} (also known as ``update'') and \textit{lexicographic revision} (also known as ``radical upgrade''), it does not excel as a learning method, because it suffers from memory limitations \cite{NinaPhD} and has a tendency to fall into cycles \cite{BaltagSmets:2011-Keep-Changing-Your-Beliefs-Aiming-for-the-Truth}. %(describe stabilisation issue). 
Given these negative results, other revision methods appear \emph{in general} preferable to minimal revision to track the truth. However, %the more natural 
the much simpler minimal revision could still be the right method  
%as successful as other methods 
when considering \textit{particular classes} of learning scenarios. Building on the framework introduced in \cite{NinaPhD,BGStruth-tracking,baltag2019truth}, we show that this is indeed the case.

After recalling the framework from \cite{baltag2019truth} in Section~\ref{sec:preliminary}, we show in Section~\ref{sec:min} that: Minimal revision can successfully learn any problem that is finitely identifiable (Section~\ref{sec:dfftmin}); It can learn on positive and negative data if the possibilities considered  are finite (Section~\ref{sec:pnmin}); Only some prior plausibility assessments can lead minimal revision to success, and we  characterize the class of such plausibility assessments when one considers finitely many possibilities (Section~\ref{sec:min_plaus}); Not all our results are preserved when learning by minimal revision with possibly erroneous observations (Section~\ref{sec:learning_errors}). In Section~\ref{sec:stronger}, we consider conditioning and lexicographic revision and characterize the prior plausibility assessments which allow one to successfully learn via these methods. We conclude in Section~\ref{sec:conclusion}.

\section{Background and Preliminary Definitions}
\label{sec:preliminary}

We consider a discrete-time iterative process ordered like the natural numbers: At each time-step, an agent receives an observation from the world, on the basis of which they revise their plausibility order and make a conjecture about which state is the actual one.  
%conjecture about which is the actual state of the world.
%among a set of possible states. 
The aim of the agent is to correctly identify the actual state of the world. Below, we recall the framework developed in \cite{NinaPhD,BGStruth-tracking,baltag2019truth}. 

\begin{definition}[\textrm{\cite[Def.~1]{baltag2019truth}}]
     An \textbf{epistemic space} $\textbf{S}$ is a pair $\langle S, O\rangle$, where $S$ is a set of possible worlds (at most countable), and $O\subseteq\mathcal{P}(S)$ is a set of observables (at most countable). For every $s\in S$, $O_s$ denotes the set of observables $p\in O$ such that $s\in p$. 
     For any two worlds $s,t\in S$, if $O_s=O_t$ then $s=t$ (no two worlds satisfy exactly the same set of observables). An epistemic space is said to be finite if $|S|<|\mathbb{N}|$.
     % , where $|\cdot|$ denotes cardinality.
     % % and $obs$ is a function assigning to every agent a method, i.e., a subset of the observables.
     \label{def:epistemic_spaces}
\end{definition}

The set of states 
%in an epistemic space 
represents the possibilities entertained by an agent, e.g.,  
%about the state of the world, e.g., 
 theories a scientist considers  possible. The set of observables represents the observations the agent might come across. 
 %and on the basis of which they produce guesses about which possible state is the actual one. 
 Figure~\ref{fig:min_not_learnable} provides an example of an epistemic space (on the left).
% three examples of epistemic spaces. 
To represent an agent's doxastic state, we enrich the space with a plausibility relation to obtain a plausibility space (Figure~\ref{fig:min_not_learnable}, right) \cite{Grove1988-GROTMF, Boutilier:1992, Baltag2016smetss}. %
% 
% In Section~\ref{sec:generalisation}, we will discuss a number of general results that hold if the assumption of totality is relaxed.

\begin{definition}[\textrm{\cite[Def.~$2$]{baltag2019truth}}]
     A \textbf{plausibility space} $\textbf{B}$ is the tuple $\langle \textbf{S},\preceq\rangle$ where $\textbf{S}=\langle S, O\rangle$ is an epistemic space and $\preceq$ is a total preorder, i.e., a total, transitive and reflexive relation over $S$. 
     \label{def:plausibility_space}
\end{definition}

\noindent For $s,t\in S$, $s\preceq t$ means that ``$s$ is at least as plausible as $t$''. We %introduce the following notational conventions: 
write $s\prec t$ if $s\preceq t \textrm{ and }t\not\preceq s$; $s\simeq t$ if $s\preceq t$ and $t\preceq s$. Since in Section~\ref{sec:stronger} we will also consider preorders that are not necessarily total, we define additional notation: $s\sim t$ if $ s\preceq t$ or $t\preceq s$; $s\not\sim t$ if not $s\sim t$.  %We denote with $\mathcal{B}$ the set of all plausibility spaces.

%Plausibility spaces are epistemic spaces equipped with a total preorder over the possible states, representing how plausible each state is assessed to be  compared to the others. 
The assumption that the preorder be total is common in the dynamic epistemic logic literature and is rooted in Grove's semantics \cite{Grove1988-GROTMF}, inspired in turn by Lewis' sphere system for counterfactuals \cite{lewis_counterfactuals_1973}. Plausibility spaces allow for a simple representation of belief as truth in the most plausible worlds \cite{Baltag2016smetss}: An agent believes a proposition $p$ iff $p$ is true in all states that the agent considers most plausible. 

\begin{definition}
Let a preorder $\preceq$ (not necessarily total) over a set of elements $S$ be given. An element $s\in S$ is a \textbf{$\preceq$-minimal element} of $S$ iff for all $t\in S$ such that $t\sim s$, $s\preceq t$. We denote by $min_\preceq X$ the set of $\preceq$-minimal elements of the set $X\subseteq S$.
% , and with $min$ the function returning the set $min_\preceq S$ on the input of a plausibility space $(S,\preceq)$. 
For every $p\in O$, we define the preorders $\preceq_p: = {\preceq} \cap (p\times p)$ and $\preceq_{\bar{p}}:={\preceq} \cap (S\setminus p \times S\setminus p)$.
\end{definition}
%Since we are considering an iterative learning process, 
The observations that an agent receives (in a step-wise manner)  form an infinite stream:
 \begin{figure}
     \centering
\scalebox{1.05}{\fbox{\begin{tikzpicture}

\coordinate (A) at (0.50,0);
\coordinate (B) at (-0.50,0);

    \draw[thick, rotate = 45] (A) ellipse (1 and 0.5) ;
    \draw[thick, rotate = -45] (B) ellipse (1 and 0.5);
\filldraw[black] (0,-0.5) circle (2pt);
\filldraw[black] (0.90,0.5) circle (2pt);
\filldraw[black] (-0.90,0.5) circle (2pt);
    \node at (1.5,0) {$q$};
\node at (-1.5,0) {$p$};
\node at (0,-0.3) {$s$};
\node at (0.90,0.25) {$t$};
\node at (-0.90,0.25) {$u$};
\end{tikzpicture}}
\hspace{0.7cm}
\fbox{\begin{tikzpicture}

\coordinate (A) at (0.50,0);
\coordinate (B) at (-0.50,0);
\draw[thick, ->] (-0.1, -0.45) -- (-0.8, 0.45) ; 
\draw[thick, ->] (0.1, -0.45) -- (0.8, 0.45) ; 
\draw[thick, ->] (-0.75, 0.54) -- (0.75, 0.54) ; 
    \draw[thick, rotate = 45] (A) ellipse (1 and 0.5) ;
    \draw[thick, rotate = -45] (B) ellipse (1 and 0.5);
\filldraw[black] (0,-0.5) circle (2pt);
\filldraw[black] (0.90,0.5) circle (2pt);
\filldraw[black] (-0.90,0.5) circle (2pt);
    \node at (1.5,0) {$q$};
\node at (-1.5,0) {$p$};
\node at (0,-0.3) {$s$};
\node at (0.90,0.25) {$t$};
\node at (-1,0.25) {$u$};
\end{tikzpicture}
\hspace{0.1cm}
\begin{tikzpicture}
    \draw[->,double,thick] (-0.7, 0.75) -- (0.7, 0.75);
    \node at (0,1.1) {\scalebox{1.4}{$p$}};
    \filldraw[black] (0,0) circle (0.01pt);
\end{tikzpicture}
\hspace{0.1cm}
\begin{tikzpicture}
\draw[thick, ->] (-0.1, -0.45) -- (-0.8, 0.45) ; 
\draw[thick, ->] (0.1, -0.45) -- (0.8, 0.45) ; 
\draw[ thick, <-] (-0.75, 0.54) -- (0.75, 0.54) ; 

\coordinate (A) at (0.50,0);
\coordinate (B) at (-0.50,0);

    \draw[thick, rotate = 45] (A) ellipse (1 and 0.5) ;
    \draw[thick, rotate = -45] (B) ellipse (1 and 0.5);
\filldraw[black] (0,-0.5) circle (2pt);
\filldraw[black] (0.90,0.5) circle (2pt);
\filldraw[black] (-0.90,0.5) circle (2pt);
    \node at (1.5,0) {$q$};
\node at (-1.5,0) {$p$};
\node at (0,-0.3) {$s$};
\node at (0.90,0.25) {$t$};
\node at (-1,0.25) {$u$};
\end{tikzpicture}}}
     \caption{On the left: An epistemic space with $S=\{u,s,t\}$, $O=\{p,q\}$. 
     On the right: a plausibility order on the same space with $t\prec u\prec s$ and the revised order $u\prec t\prec s$ after observing $p$ (represented by the $p$-labelled arrow). This  is an example of (one-step) \textit{minimal revision} (Def.~\ref{def:min_update}). An arrow from one world to another indicates that the latter is  more plausible than the former. We omit reflexive arrows.}
     \label{fig:min_not_learnable}
 \end{figure}
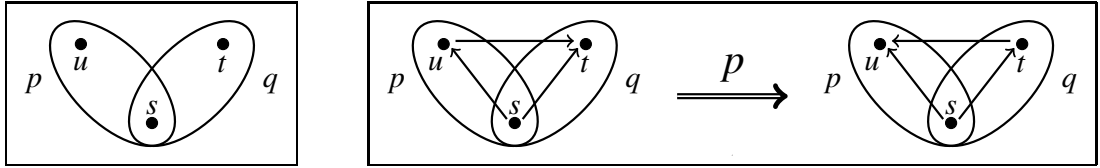
\begin{definition}[\textrm{\cite[Def.~4, 5]{baltag2019truth}}]
%Notations for data streams and data sequences
    A \textbf{data sequence} $\sigma$ is a finite sequence of elements of $O$; a \textbf{data stream} $\vec{O}$ is an infinite sequence of elements in $O$.
    %$O\cup\{\top\}$; a \textbf{data stream} $\vec{O}$ is an infinite sequence of elements in $O\cup\{\top\}$. 
    Given a data sequence $\sigma=(O_0,O_1,\ldots, O_n)$, we denote by $set(\sigma)$ the set of elements of $O$ that occur in $\sigma$; given a data stream $\vec{O}=(O_0,O_1,\ldots)$, we let $set(\vec{O})$ be the set of elements of $O$ that occur in $\vec{O}$. Given $n\in \Nat$, we denote by $O_n$ the $n^{th}$ observation in $\vec{O}$.  $\vec{O}[k]$ is the initial segment of $\vec{O}$ of length $k$. $\sigma*\vec{O}$ is the concatenation of the sequence $\sigma$ with the data stream $\vec{O}$. The empty sequence is denoted by $\epsilon$.
    %For $\sigma=(\sigma_0,\sigma_1,\sigma_2...\sigma_n)$, $len(\sigma)=n+1$.
\end{definition}

We consider  streams containing all and only observations that are true in the actual world. 
% Complete and sound data streams are also known as texts \cite{jain1999systems} or environments \cite{martin1998elements}. 
%We will consider other learning scenarios in Section~\ref{sec:generalisation}.

\begin{definition}[\textrm{\cite[Defs.~6,~7]{baltag2019truth}}] A \textbf{data sequence} $\sigma$ is: \textbf{sound with respect to a state} $s$ iff $set(\sigma)\subseteq O_s$; \textbf{complete with respect to a state} $s$ iff $O_s\subseteq set(\sigma)$. A \textbf{data stream} $\vec{O}$ is: \textbf{sound with respect to a state} $s$ iff $ set(\vec{O})\subseteq O_s$; \textbf{complete with respect to a state} $s$ iff $ O_s\subseteq set(\vec{O})$.  
\end{definition}
%Complete data stream and data sequence 

At each time-step, a learner uses some belief revision method to revise their plausibility order on the basis of the new observation received. Formally:

\begin{definition}[\textrm{\cite[Def.~$10$]{baltag2019truth}}]
    A \textbf{one-step belief revision} is a function $R_1$ taking a plausibility space $\textbf{B}$ and an observable $p\in O$ to output the plausibility space $R_1(\textbf{B}, p)$.
    A one-step belief revision function $R_1$ induces an \textbf{iterated belief revision function} $R$  in the following way: $R(\textbf{B},\epsilon)=\textbf{B}$, $ R(\textbf{B},\sigma*p)=R_1(R(\textbf{B},\sigma),p)$.

\end{definition}

Any belief-revision method can be used to define a corresponding learning method.

\begin{definition}[\textrm{\cite[Def.~$8$]{baltag2019truth}}]
     A \textbf{learning method} $L$ is a function that takes as input an epistemic space $\textbf{S}=\langle S, O \rangle$ and a data sequence $\sigma$ and outputs a conjecture $L(\textbf{S},\sigma)\subseteq S$. 
\end{definition}

\begin{definition}[\textrm{\cite[Def.~$11$]{baltag2019truth}}]
A \textbf{plausibility assignment} $\textbf{PLAUS}$ is a function that takes as input an epistemic space $\textbf{S}=\langle S, O\rangle$ and assigns to it a plausibility order $\preceq$, converting it into a plausibility space $\textbf{PLAUS}(S)=\langle\textbf{S},\preceq\rangle$.
\end{definition}

\begin{definition}[\textrm{\cite[Def.~$12$]{baltag2019truth}}]
Let a belief-revision method $R$ be given. Given any epistemic space $\textbf{S}$ and plausibility assignment $\textbf{PLAUS}$, a belief-revision method $R$ induces a \textbf{canonical learning method} in the following way: $L^{PLAUS}_R (\textbf{S}, \sigma):= min \textrm{ } R(PLAUS(\textbf{S}),\sigma)$.

\end{definition}

\noindent Given an epistemic space $\textbf{S}$, a plausibility assignment \textbf{PLAUS}, a one-step belief-revision method $R_1$, and $\vec{O}[n]$, we denote with $\preceq^{R_1}_{\vec{O}[n]}$ the plausibility order in the updated epistemic space  $R(\textbf{PLAUS}(S),\vec{O}[n])$.
% We will now specify when a learner has successfully learned the correct state. 

As the central criterion of successful learning we take \emph{identifiability in the limit}, which originally comes from the studies of inductive grammar inference \cite{GOLD1967447}, and is often combined with computational restrictions on learners and the classes they learn. In this paper, we abstract away from such constraints and focus on purely structural considerations. 

\begin{definition}[\textrm{\cite[Def.~$18$]{baltag2019truth}}, originally in \cite{GOLD1967447}] 

Let an epistemic space $\textbf{S}=\langle S, O\rangle$ be given.
A \textbf{state }$s\in S$ is \textbf{identified in the limit by a learning method}\textbf{ $L$} iff for any sound and complete data stream $\vec{O}$ for $s$, there is an $n\in \mathbb{N}$ such that for all $k\geq n$, $L(\textbf{S},\vec{O}[k])=\{s\}$.
The \textbf{epistemic space} $\textbf{S}$ \textbf{is identified in the limit by learning method $L$} iff all $s\in S$ are identified in the limit by learning method $L$.
The \textbf{epistemic space} $\textbf{S}$ is \textbf{identifiable in the limit} (\textbf{learnable}) iff it is identified in the limit by some learning method.
\label{def:identifiability_limit}
\end{definition}

\noindent Instead of such general, unrestricted learning functions, we will focus on more constructive ones, namely, those guided by some underlying belief revision methods.
\begin{definition}[\textrm{\cite[Def.~$19$]{baltag2019truth}}]
An \textbf{epistemic space} $\textbf{S}$ \textbf{is identifiable in the limit by the belief revision method} $R$ if there exists a prior plausibility assignment $PLAUS$ such that $L^{PLAUS}_R$ identifies $S$ in the limit.
A \textbf{belief revision method} $R$ is \textbf{universal on a class $C$} of epistemic spaces if every space in $C$ is identifiable by the belief revision method. A \textbf{belief revision method} $R$ is \textbf{universal (tout court)} if it is universal in the class of all learnable spaces.
\end{definition}

\section{Learning with Minimal Revision}
\label{sec:min}

What is the least disruptive way of revising plausibility spaces? Such a minimal policy was first discussed in \cite{Spohn1988} and then generalised to total preorders in \cite{Boutilier:1992}. Upon receiving  $p$, minimal revision moves all the most plausible $p$-worlds in front of all other worlds, leaving the rest of the plausibility order unchanged. Minimal revision is \textit{conservative} in the sense that it leaves as much as possible of the old plausibility ordering unchanged \cite{Boutilier:1992,NinaPhD}.\footnote{For a systematic analysis of the learning-theoretic properties of the canonical learning methods induced by minimal revision and other belief revision methods, see \cite{NinaPhD}.} In our setting, it is defined in the following way. %This question was addressed in \cite{Boutilier:1992}. The minimal (aka natural) belief revision policy accepts the new information, but preserves as much as possible of the old plausibility structure.    
\begin{definition}[Minimal Revision, \cite{Baltag2016smetss}]
    Let $\textbf{B}=\langle S, O, \preceq \rangle$ be a plausibility space and $p\in O$. The one-step revision function $\textbf{mini}(\textbf{B}, p)$ generates the plausibility space $\textbf{B}'=\langle S, O, \preceq' \rangle$, where $\preceq'$ is such that if $s\in min_{\preceq} p$ and $t\not\in min_{\preceq} p$, then $s\preceq' t$ and $t\not\preceq' s$; $s\preceq' t$ iff $s\preceq t$, otherwise.
    \label{def:min_update}
\end{definition}
 
  Minimal revision suffers from a number of shortcomings. For instance, it can fail to stabilize when revising with true higher-order information \cite{BaltagSmets:2011-Keep-Changing-Your-Beliefs-Aiming-for-the-Truth}, and it fails to learn the least complex problems in the learning setting of \cite{kelly1999iterated}. 
 With respect to the framework adopted in this paper, it has been shown that minimal revision is not universal, in the sense that not all spaces that are learnable can be learned via minimal revision \cite{NinaPhD,BGStruth-tracking,baltag2019truth}.
 Consider the space in Figure~\ref{fig:min_not_learnable} from \cite{baltag2019truth}. First, note that the space is learnable by the cruder \textit{conditioning} method (Definition~\ref{def:conditioning}): At each time-step, eliminate all worlds at which the new information is false. In this case, any initial plausibility order would do, as long as $u\prec s$ and $t\prec s$. Yet, there is no plausibility assignment %for the space in Figure~\ref{fig:min_not_learnable} 
 allowing identification of this space in the limit via minimal revision: In order to identify either $u$ or $t$, both of them must be strictly more plausible than $s$; but then $s$ is not identified, since no observation will ever suffice to make $s$ strictly more plausible than $u$ and $t$.
 %, since both $p$ and $q$ are consistent with the more plausible worlds $u$ and $t$. 

We are now ready to state our main research question: Even if minimal revision cannot identify all learnable epistemic spaces, are there interesting classes of epistemic spaces that it can identify? In the sections below, we give a positive answer to this question and consider some notable classes of spaces in which minimal revision is universal. All omitted proofs are included in the \href{sec:Appendix}{Appendix}.

\subsection{Minimal Revision is Universal on the Class of Finitely Identifiable Spaces} \label{sec:dfftmin}

The success criterion of identifiability in the limit in Definition~\ref{def:identifiability_limit} is often contrasted with a stronger criterion of conclusive learning (see \cite{gierasimczuk_dejong_2012, degremont2011finite}), adapted from finite identifiability as proposed in \cite{Freivald:1979aa,Klette:1980aa,Jantke:1981aa}. Finite identifiability was characterized in \cite{mukouchi1992characterization}, studied, e.g., in \cite{lange_zeugmann,vargas2020path}, and applied to problems in AI, e.g., in \cite{BG2015,BG2018,Singleton:2024aa}. It requires not only that the learner \emph{converges} to a correct conjecture, but also that they know with certainty that the conjecture is true. The latter condition of certainty can be rendered as the requirement that for any stream the learner is a once-defined function, i.e., throughout the learning process it only has one-shot at the correct guess \cite{gierasimczuk_dejong_2012}. In this section, we will show that \textbf{mini} can identify in the limit the class of epistemic spaces that can be finitely identified, i.e., those spaces for which there exists a learner that can identify the actual world with a single guess. 

 \begin{figure}
     \centering
\scalebox{1}{\begin{tikzpicture}
    \coordinate (A) at (0,0);
\coordinate (B) at (-0.50,0);
\coordinate (C) at (0,0.55);

    % \draw[thick] (A) ellipse (1.7 and 0.7) ;
    \draw[thick, rotate = 30] (-0.27,0.87) ellipse (0.7 and 1.7) ;
    \draw[thick,rotate = -30] (0.27,0.87) ellipse (0.7 and 1.7) ;
    \draw[thick] (0,1.8) ellipse (1.7 and 0.7) ;
    % \draw (-0.90, 0) circle (0.5);
    
    %\draw[thick, ->] (-0.7, 0) -- (0.7, 0);

\filldraw[black] (1.3,1.8) circle (2pt);
\filldraw[black] (-1.3,1.8) circle (2pt);
\filldraw[black] (0,-0.4) circle (2pt);
% \filldraw[black] (-0.90,0) circle (2pt);
    \node at (1.8,0.6) {$q$};
    \node at (0,2.8) {$r$};
    \node at (-1.8,0.5) {$p$};
    % \node at (-2,1.8) {$p'$};
   
% \node at (-1.6,0) {$p$};

\node at (0,-0.2) {$s$};
\node at (-1.2,1.6) {$u$};
\node at (1.2,1.6) {$t$};
\end{tikzpicture}
\hspace{1cm}
\begin{tikzpicture}
    \coordinate (A) at (0,0);
\coordinate (B) at (-0.50,0);
\coordinate (C) at (0,0.55);

    \draw[thick] (A) ellipse (1.7 and 0.7) ;
    \draw[thick] (-0.9,0.9) ellipse (0.7 and 1.7) ;
    \draw[thick] (0.9,0.9) ellipse (0.7 and 1.7) ;
    \draw[thick] (0,1.8) ellipse (1.7 and 0.7) ;
    % \draw (-0.90, 0) circle (0.5);
    
    %\draw[thick, ->] (-0.7, 0) -- (0.7, 0);

\filldraw[black] (0.90,1.8) circle (2pt);
\filldraw[black] (-0.90,1.8) circle (2pt);
\filldraw[black] (0.90,0) circle (2pt);
\filldraw[black] (-0.90,0) circle (2pt);
    \node at (-1,2.8) {$q$};
    \node at (1,2.9) {$\bar{q}$};
    \node at (-2,0) {$p$};
    \node at (-2,1.8) {$\bar{p}$};
   
% \node at (-1.6,0) {$p$};

\node at (0.70,0.3) {$s$};
\node at (-0.70,0.3) {$t$};
\node at (-0.70,2) {$u$};
\node at (0.70,2) {$w$};
\draw[thick] (-0.70, 0.3) circle (0);
\end{tikzpicture}
\hspace{1.15cm}
\begin{tikzpicture}
    \filldraw[black] (-0.70,2) circle (2pt);
    \filldraw[black] (0.70,2) circle (2pt);
    \filldraw[black] (2.1,2) circle (2pt);
    \draw[opacity=0.7] (3.5,2) circle (2pt);
    \draw[thick] (0, 2) ellipse (1.2 and 0.7) ;
    \draw[thick] (1.4, 2) ellipse (1.2 and 0.7) ;
    \draw[thick] (2.8, 2) ellipse (1.2 and 0.7) ;
    \draw[thick, dashed,rotate=180] (-4.2,-2) +(90:1.2 and 0.7) arc (90:-90:1.2 and 0.7);
    % (4.2, 0) ellipse (1.2 and 0.7) ;
    \draw[thick] (-0.70, 2) circle (0.4);
    \draw[thick] (-0.70, 0.3) circle (0);
    \node at (-0.70,2.2) {$s_0$};
    \node at (0.70,2.2) {$s_1$};
    \node at (2.1,2.2) {$s_2$};
    \node at (0,3) {$p_1$};
    \node at (-0.25,1.6) {$p_0$};
    \node at (1.4,3) {$p_2$};
    \node at (2.8,3) {$p_3$};
\end{tikzpicture}}

     \caption{Example of three epistemic spaces that are finitely identifiable. On the left: $S=\{s,u,t\}$, $O=\{p,q,r\}$ . At the center: $S=\{s,u,t,w\}$, $O=\{p,\bar{p},q,\bar{q}\}$. On the right: $S=\{s_i|i\in N\}$, $O=\{p_i:i\in \Nat\}$, and for all $s_i$, $O_i=\{p_j\in O: j=i\textrm{ or }j=i+1\}$.}
     \label{fig:dfttspaces}
 \end{figure}
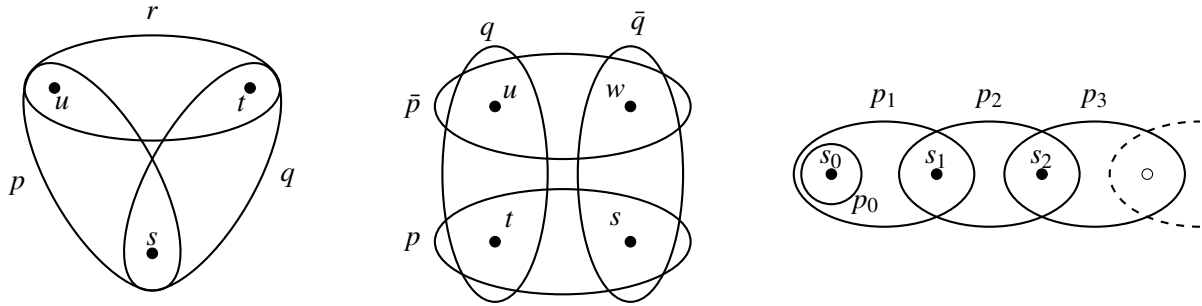

\begin{definition}[\cite{mukouchi1992characterization}]\label{def:fin_id}
Let an epistemic space $\textbf{S}=\langle S, O\rangle$ be given.
A\textbf{ state} $s\in S$ is \textbf{finitely identified} by a learning method $L$ iff for any sound and complete data stream $\vec{O}$ for $s$, when inductively given $\vec{O}$, $L$ outputs at some point a single conjecture ${s}$. 
The \textbf{epistemic space} $\textbf{S}$ is \textbf{finitely identified} by learning method $L$ iff all $s\in S$ are finitely identified by learning method $L$.
The \textbf{epistemic space} $\textbf{S}$ is \textbf{finitely identifiable} (\textbf{conclusively learnable}) just in case it is finitely identified by some learning method.
\end{definition}

In \cite{mukouchi1992characterization}, in a computable setting, it has been shown that this class of spaces is characterized by the existence of the following map:

 \begin{definition}[\textrm{\cite[Def.~$10.8$]{gierasimczuk2014logic}}, originally in \cite{mukouchi1992characterization}] Let $\textbf{S}=\langle S, O\rangle$ be an epistemic space, and let $s\in S$ be given. A \textbf{definite finite tell-tale map} is a total map assigning to each $
s\in S$ a \textbf{definite finite tell-tale set} $D_s\subseteq O$ such that: (i) $D_s$ is finite; (ii) $D_s\subseteq O_s$; (iii) for any $t\in S$, if $D_s\subseteq O_t$, then $s=t$.
\label{def:dfttm}
\end{definition}

\noindent Informally, the existence of a definite finite tell-tale map on a space means that each world is uniquely identified by some finite conjunction of observables. Fig.~\ref{fig:dfttspaces} shows examples of spaces with definite finite tell-tale maps. The characterisation below adapts to our setting the original theorem in \cite{mukouchi1992characterization}. 

\begin{proposition}\label{th:finite_id}
    An epistemic space $\textbf{S}$ is finitely identifiable iff it has a definite finite tell-tale map.
\end{proposition}

 We show now that \textbf{mini} can learn the class of finitely identifiable spaces, i.e., it can identify this class in the limit. We do this by showing that \textbf{mini} can learn any space that has a definite finite tell-tale map. We start with the following Lemma:

\begin{lemma}\label{lemma:mini_persistence}
    Let an epistemic space $\textbf{S}=\langle S, O\rangle$ be given. Let \textbf{PLAUS} be the plausibility assignment that assigns to $\textbf{S}$ the preorder $\preceq$ such that for all $s,t\in S$, $s\simeq t$. Then, for any state $s$ and for any sound data-stream $\vec{O}$ for $s$, $s\in L^{\textbf{PLAUS}}_\textbf{mini}(\textbf{S}, \vec{O}[n])$ for all $n\in \mathbb{N}$.
\end{lemma}
\begin{proof}
 We prove by induction on $n$ that $s\in min_{\preceq^{\textbf{mini}}_{\vec{O}[n]}}$.
    [Base Case: $n=0$] The sequence $\vec{O}[0]$ is empty, and as such the plausibility model we consider is the initial plausibility model. Since for all $t\in S$, $s\preceq t$, we have that $s$ is a minimal element of $\preceq$.
    [Inductive Step] Assume as inductive hypothesis that for all $m\leq n$, it is the case that $s\in min_{\preceq^{\textbf{mini}}_{\vec{O}[m]}}$. Consider now $n+1$, and consider the observable $p$ occurring in the position $n+1$ in the stream $\vec{
    O}$, i.e., the observable after the sequence $\vec{O}[n]$. By assumption, $\vec{O}$ is sound with respect to $s$, and thus $s\in p$. In addition, by inductive hypothesis we have that $s\in min_{\preceq^{\textbf{mini}}_{\vec{O}[n]}}p$. Hence, $s\in min_{\preceq^{\textbf{mini}}_{\vec{O}[n + 1]}}$ by Def.~\ref{def:min_update}. Therefore, in any sound and complete data-stream $\vec{O}$ for all $n\in \mathbb{N}$, $s\in min_{\preceq^{\textbf{mini}}_{\vec{O}[n]}}$, and thus $s\in L^\textbf{PLAUS}_{\textbf{mini}}(\textbf{S}, \vec{O}[n])$.
\end{proof}

\begin{theorem}\label{prop:at_least_DFTT_mini}
Minimal revision is universal on the class of spaces that are finitely identifiable. 
\end{theorem}
\begin{proof}
    Let $\textbf{S}$ be an epistemic space that is finitely identifiable. Then, by Prop.~\ref{th:finite_id}, $\textbf{S}$ has a definite finite tell-tale map. Let the plausibility order $\preceq$ be given such that for all $t,s\in S$, $s\simeq t$. Assume towards a contradiction that the learner induced by $\textbf{mini}$ and $\preceq$ does not identify the space $S$, i.e., there is a state $s\in S$ and a sound and complete data-stream $\vec{O}$ for $s$ such that there is no $n\in\mathbb{N}$ such that for all $k\geq n$, $L^{\textbf{PLAUS}}_\textbf{mini}(\textbf{S}, \vec{O}[k])=\{s\}$.
    Consider such a data-stream. Since $\vec{O}$ is sound and complete with respect to $s$, and since there exists a definite finite tell-tale set $D_s$ for $s$, there is some $n\in \mathbb{N}$ such that $D_s\subseteq set(\vec{O}[n])$. Consider such a sequence $\vec{O}[n]$. Since $s$ is not identifiable in the limit, there exists a $k\geq n$, such that $L^{\textbf{PLAUS}}_\textbf{mini}(\textbf{S}, \vec{O}[k])\not=\{s\}$. By Lemma~\ref{lemma:mini_persistence}, it must be the case that $s\in L^{\textbf{PLAUS}}_\textbf{mini}(\textbf{S}, \vec{O}[k])$. Therefore, there must exist a $t\not= s$ such that $t\in L^{\textbf{PLAUS}}_\textbf{mini}(\textbf{S}, \vec{O}[k])$. It follows that $t\in \bigcap set(\vec{O}[k])$ (otherwise, it would have become strictly less plausible than $s$ for some $k'\leq k$).  Since $D_s\subseteq set(\vec{O}[k])$, we have $t\in\bigcap D_s$. By Def.~\ref{def:dfttm}, we have $t=s$. Contradiction.
\end{proof}

The theorem above states that any space that is finitely identifiable can be learned via \textbf{mini} by using an initially `unbiased' plausibility order, which considers all possibilities equally plausible. This plausibility assignment together with \textbf{mini} induces a method that patiently waits for enough data to dismiss all but one possibility. 
Note, however, that starting from a prior that ranks all states as equally plausible is not necessary for learning finitely identifiable spaces via \textbf{mini}, i.e., there are finitely learnable spaces that can be learned via \textbf{mini} and a plausibility assignment where for at least two states $s,t$, $s\prec t$. For instance, the space on the right in Figure~\ref{fig:dfttspaces} can be identified by ranking the states in a way such that $s_i\preceq s_j$ iff $i\leq j$.

Let us also note that \textbf{mini} can identify spaces that are not finitely identifiable and thus the class of finitely identifiable epistemic spaces is properly included in the class of spaces on which \textbf{mini} is universal. For instance, the epistemic space with $S=\{s,t\}$, $O=\{p,q\}$, $O_s=\{p,q\}$ and $O_t=\{q\}$ is a space identifiable via \textbf{mini} using the plausibility order $t\prec s$. This space is not finitely identifiable since $O_t\subseteq O_s$ and therefore $t\in S$ has no definite finite tell-tale set.
% and thus  which does not have  Note also that the class of spaces \textbf{Show that there is a space that is space that is identifiable by mini but is not fin-id, so the proper inclusion $FIN\subset MINI$.}

\subsection{Minimal Revision is Universal on Positive and Negative Data on Finite Spaces} 
\label{sec:pnmin}

In the inductive inference literature, it is common to distinguish between learning from only positive information and from  positive and negative information. This distinction is given a lot of attention, because the latter setting is much more informative, and so it leads to different learning powers of the same learning methods (see, e.g., \cite{Angluin:1983aa}). In our case, learning from positive and negative information means that for every observable its complement can also be observed. In logic parlance, this is expressed as closure under negation (i.e., if $p$ can be observed, so can $\neg p$). We will therefore speak of epistemic spaces that are closed under negation \cite{baltag2019truth}. Negation-closed spaces are a subclass of the class of \textit{strongly separated} spaces \cite{baltag2016solvability}. 

\begin{definition}
    An epistemic space $\textbf{S}=\langle S,O\rangle$ is: \textbf{strongly separated} when for every $s,t\in S$, $O_s\not\subseteq O_t$; \textbf{negation-closed} iff for every $p\in O$, there is $\bar{p}\in O$ such that $p=S\setminus \bar{p}$.
\end{definition}

 Figure~\ref{fig:dfttspaces} shows an example of a strongly separated space that is not negation-closed (left) and an example of a negation-closed space (center).
From \cite{NinaPhD}, we know that \textbf{mini} is not universal on positive and negative data on the class of all negation-closed spaces, and thus the same holds for the larger class of strongly separated spaces. Nevertheless, we can show that 
%when restricted to finite spaces, 
 \textbf{mini} is in fact universal on the class of \textit{finite} strongly separated spaces and thus on the class of \textit{finite} negation-closed spaces. 
\begin{lemma}
    Let $\textbf{S}$ be a finite strongly separated epistemic space. Then $\textbf{S}$ has a definite finite tell-tale map.\label{lemma:negation_closedtodfft}
\end{lemma}
\begin{proof}
    Consider a strongly separated epistemic space $\textbf{S}$. Take an arbitrary$s\in S$. Consider the set $D^s:=\bigcup_{s'\in S\setminus\{s\}}(O_s\setminus{O_{s'}})$.
    % and any enumeration $s_0,s_1,..., s_i,..., s_{n}$ of all elements in the set $S\setminus\{s\}$. Construct the following set: $D^s_0:= O_s\setminus O_{s_0}$; $D^s_{i +1}:=D^s_i\cup (O_s\setminus O_{s_i+1})$. 
    Since the space is strongly separated, for all $s'\in S$ such that $s'\not=s$, $O_s\setminus O_{s'}\not=\emptyset$. Thus for all $s'\not= s$, $s'\not\in \bigcap D^s$. Since the space is finite, $D^s$ is finite, and $D^s\subseteq O_s$ by construction. By Def.~\ref{def:dfttm}, $D^s$ is a definite finite tell-tale set for $s$. Since $s$ was arbitrary, $S$ has a definite finite tell-tale map.
\end{proof}

\begin{theorem}
Minimal revision is universal on the class of finite strongly separated epistemic spaces.
    \label{th:separated_mini}
\end{theorem}

\begin{proof}
    By Lemma~\ref{lemma:negation_closedtodfft} and Prop.~\ref{th:finite_id}, the class of finite strongly separated epistemic spaces is a subset of the class of finitely identifiable spaces. By Th.~\ref{prop:at_least_DFTT_mini}, \textbf{mini} is universal on the class of finitely identifiable spaces, and thus on the class of finite strongly separated spaces.
\end{proof}

\begin{corollary}
Minimal revision is universal on positive and negative data on the class of finite negation-closed epistemic spaces.
    \label{th:negation_closed}
\end{corollary}

Therefore, one can safely learn from positive and negative data using \textbf{mini} when only finitely many states are possible. Let us note here that this result has been foreshadowed by Boutilier in his seminal work \cite{Boutilier:1992}, where minimal revision is applied to finitely-grounded knowledge sets over propositional logic (which is closed on negation).

\subsection{What Priors for Minimal Revision?} 
\label{sec:min_plaus}

We have considered above some classes of spaces where \textbf{mini} is universal. However, in those and all other classes on which \textbf{mini} is universal, only some plausibility orders will allow learning via \textbf{mini}. We call these orders \emph{appropriate}: 
%only some plausibility orders in combination with \textbf{mini} will induce a canonical learning method that identifies in the limit the given space. In other words, not all prior plausibility assignments are \textit{appropriate} for learning a given space via \textbf{mini}. 

 \begin{definition}
    Let $\textbf{S}$ be an epistemic space. A \textbf{plausibility order} $\preceq$ is \textbf{appropriate for learning $\textbf{S}$ via the belief revision method $R$} if the canonical learning method induced by $R$ and the plausibility assignment assigning $\preceq$ to $S$ identifies $\textbf{S}$ in the limit.  
    \label{def:appropriateness}
\end{definition}

Figure~\ref{fig:bad_prior_mini} shows a plausibility order that is not appropriate for learning a space via \textbf{mini}: The state $s$ (in red) is not identified with the (sound and complete) data-stream in which  $q$ occurs once first, followed by an infinite stream of $p$s. With this stream, the state $u$ will forever remain equi-plausible to $s$, and thus $s$ will never be the uniquely minimal element of the space. 

Below, we give a characterisation of the class of appropriate plausibility orders for learning via \textbf{mini} on \textit{finite} epistemic spaces. Our focus on finite spaces is justified by the fact that they are often sufficient for a variety of applications in belief revision theory, knowledge representation, and dynamic epistemic logic (where possible worlds are identified with valuations over a finite set of propositions, see, e.g., \cite{DEL}). Furthermore, the full characterisation of the class of spaces for which \textbf{mini} is universal is still unknown, so we cannot resort to knowledge about the underlying structure of the epistemic space in our analysis.%\textbf{mini} is sensitive to the order in which information is received \cite{protocols_belief_merge, NinaPhD}; and a characterisation of . .

 Let us remark that our focus here is different from that in \cite{baltag2019truth}, where the focus is on whether a belief revision method is universal in the sense that for any learnable space one can construct \textit{some} plausibility order that guarantees successful learning via that method. The interest in \cite{baltag2019truth} therefore lies in the existence of a suitable plausibility order on every learnable space. In contrast, we aim at characterising for \textbf{mini} (and later in Section~\ref{sec:stronger} also for conditioning and lexicographic upgrade) the class of \textit{all} prior plausibility orders that are suited for learning every space that is learnable via the method under consideration. Of course, an appropriate order for a method can only exist on spaces that can be learned via that method, and thus \cite{baltag2019truth} provides the limits of applicability of our analysis.
 Let us start with some preliminary lemmas.

 \begin{lemma}
    If a plausibility order $\preceq$ is appropriate to learn an epistemic space $\textbf{S}$ via \textbf{mini}, then for all $s\in S$, there exists a $p\in O_s$ such that $s\in min_{\preceq} p$.
    \label{lemma:atleastonemin}
\end{lemma}

\begin{proof}
    Assume that $\preceq$ is appropriate to learn $\textbf{S}$ via \textbf{mini}. Assume towards a contradiction that there exists an $s$ such that for all $p\in O_s$, $s\not\in min_{\preceq} p$. 
    Consider now a sound and complete data stream such that it enumerates all the elements in $O_s$. We show by induction on the length $n$ of any finite segment of $\vec{O}[n]$ that for all $p\in O_s$ there exists a $t$ such that $t\prec^{\textbf{mini}}_{\vec{O}[n]} s$ and $t\in p$. 
   [Base case $n=0$] Follows directly from the assumption. [Inductive step: $n= m+1$] Assume that for $O[m]$, for all $p\in O_s$ there is a $t$ such that $t\in p$ and $t\prec^{\textbf{mini}}_{O[m]} s$. Consider now $p\in O_s$ that occurs in $\vec{O}$ after $O[m]$. By inductive hypothesis, it follows that there is a $t\in S$ such that $t\prec^{\textbf{mini}}_{O[m]}s$ and $t\in p$. Hence $t\prec^{\textbf{mini}}_{O[m+1]}s$. Hence $s\not\in min_{\preceq^{\textbf{mini}}_{\vec{O}[m+1]}}$.
\end{proof}

\noindent The above lemma expresses that an appropriate plausibility order guarantees that in every state it is possible to make an observation that `favors' that state.

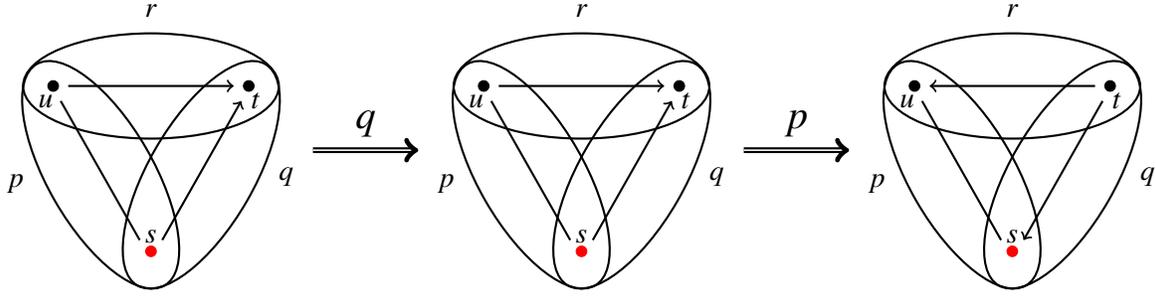
\begin{figure}
     \centering
     \hspace{-0.2cm}
\scalebox{1}{\begin{tikzpicture}
    \coordinate (A) at (0,0);
\coordinate (B) at (-0.50,0);
\coordinate (C) at (0,0.55);

    % \draw[thick] (A) ellipse (1.7 and 0.7) ;
    \draw[thick, rotate = 30] (-0.27,0.87) ellipse (0.7 and 1.7) ;
    \draw[thick,rotate = -30] (0.27,0.87) ellipse (0.7 and 1.7) ;
    \draw[thick] (0,1.8) ellipse (1.7 and 0.7) ;

\filldraw[black] (1.3,1.8) circle (2pt);
\filldraw[black] (-1.3,1.8) circle (2pt);
\filldraw[red] (0,-0.4) circle (2pt);
\draw[->, thick] (-1.1, 1.8) -- (1.1,1.8);
\draw[->, thick] (0.15,-0.25) -- (1.2,1.6);
\draw[-, thick] (-0.15,-0.25) -- (-1.2,1.6);

    \node at (1.8,0.6) {$q$};
    \node at (0,2.8) {$r$};
    \node at (-1.8,0.5) {$p$};

\node at (0,-0.2) {$s$};
\node at (-1.4,1.6) {$u$};
\node at (1.4,1.6) {$t$};
\end{tikzpicture}
\hspace{-0.1cm}
\begin{tikzpicture}
    \draw[->,double,thick] (-0.7, 2) -- (0.7, 2);
    \node at (0,2.35) {\scalebox{1.4}{$q$}};
    \filldraw[black] (0,0) circle (0.01pt);
\end{tikzpicture}
\hspace{-0.1cm}
\begin{tikzpicture}
    \coordinate (A) at (0,0);
\coordinate (B) at (-0.50,0);
\coordinate (C) at (0,0.55);

    \draw[thick, rotate = 30] (-0.27,0.87) ellipse (0.7 and 1.7) ;
    \draw[thick,rotate = -30] (0.27,0.87) ellipse (0.7 and 1.7) ;
    \draw[thick] (0,1.8) ellipse (1.7 and 0.7) ;

\filldraw[black] (1.3,1.8) circle (2pt);
\filldraw[black] (-1.3,1.8) circle (2pt);
\filldraw[red] (0,-0.4) circle (2pt);
\draw[->, thick] (-1.1, 1.8) -- (1.1,1.8);
\draw[->, thick] (0.15,-0.25) -- (1.2,1.6);
\draw[-, thick] (-0.15,-0.25) -- (-1.2,1.6);

    \node at (1.8,0.6) {$q$};
    \node at (0,2.8) {$r$};
    \node at (-1.8,0.5) {$p$};

\node at (0,-0.2) {$s$};
\node at (-1.4,1.6) {$u$};
\node at (1.4,1.6) {$t$};
\end{tikzpicture}
\hspace{-0.1cm}
\begin{tikzpicture}
    \draw[->,double,thick] (-0.7, 2) -- (0.7, 2);
    \node at (0,2.35) {\scalebox{1.4}{$p$}};
    \filldraw[black] (0,0) circle (0.01pt);
\end{tikzpicture}
\hspace{-0.1cm}
\begin{tikzpicture}
    \coordinate (A) at (0,0);
\coordinate (B) at (-0.50,0);
\coordinate (C) at (0,0.55);

    \draw[thick, rotate = 30] (-0.27,0.87) ellipse (0.7 and 1.7) ;
    \draw[thick,rotate = -30] (0.27,0.87) ellipse (0.7 and 1.7) ;
    \draw[thick] (0,1.8) ellipse (1.7 and 0.7) ;

\filldraw[black] (1.3,1.8) circle (2pt);
\filldraw[black] (-1.3,1.8) circle (2pt);
\filldraw[red] (0,-0.4) circle (2pt);
\draw[<-, thick] (-1.1, 1.8) -- (1.1,1.8);
\draw[<-, thick] (0.15,-0.25) -- (1.2,1.6);
\draw[-, thick] (-0.15,-0.25) -- (-1.2,1.6);

    \node at (1.8,0.6) {$q$};
    \node at (0,2.8) {$r$};
    \node at (-1.8,0.5) {$p$};
\node at (0,-0.2) {$s$};
\node at (-1.4,1.6) {$u$};
\node at (1.4,1.6) {$t$};
\end{tikzpicture}}
     \caption{Example of a plausibility order that is not appropriate to learn a space via \textbf{mini}. The epistemic space is given by $S=\{s,u,t\}$ and $O=\{p,q,r\}$; the prior plausibility order (leftmost space) is $t\prec s$, $t\prec u$, $u\simeq s$. The  labelled arrows represent two revision steps, first when observing $q$, and then $p$. 
     %The sequence of spaces represents the initial sequence of revision by \textbf{mini} on a data stream that starts with $q$ (first labelled arrow) and reports only $p$ afterwards.
     }
     \label{fig:bad_prior_mini}
 \end{figure}

\begin{lemma}
    Suppose that a plausibility order $\preceq$ is appropriate to learn an epistemic space $\textbf{S}$ via \textbf{mini}. For all $s\in S$, for all $t\in S$ such that $s\simeq t$ and $t\not=s$, if $|\{p\in O_s: t\not\in p\}|<|\mathbb{N}|$, then there exists a $p\in O_s$ such that $s\in min_{\preceq} p$ and $t\not\in p$.
\label{lemma:infinity_necessary_mini}
\end{lemma}

\begin{proof}
    Towards a contradiction, assume that there are $s,t\in S$ such that: $s\simeq t$, $t\not= s$; $|\{p\in O_s: t\not\in p\}|<|\mathbb{N}|$; for all $p\in O_s$ if $s\in min_{\preceq}p$ then $t\in p$. Let $\vec{O}$ be a stream 
    %such that there is an 
    and $n\in\mathbb{N}$ be such that $set(\vec{O}[n])=\{p\in O_s: t\not\in p\}$, which 
    %we can construct 
    exists since $|\{p\in O_s: t\not\in p\}|<|\mathbb{N}|$. %Let us denote 
    Denote with $\preceq^{\textbf{mini}}_{\vec{O}[n]}$ the order in %the updated space 
    $\textbf{mini}(\textbf{PLAUS}(\textbf{S}), \vec{O}[n])$. Since by assumption for all $p\in O_s$ if $t\not\in p$ then $s\not\in min_\preceq p$, and since by construction for all $p\in set(\vec{O}[n])$, $t\not\in p$, we have that $s\not\in min_{\preceq}p$ for all $p\in set(\vec{O}[n])$. It follows that $s\simeq^{\textbf{mini}}_{\vec{O[n]}}t$ (otherwise one of them would be a minimal element of some $p\in \vec{O}[n]$). 
    Consider now any extension of the sequence $\vec{O}[n]$ that enumerates all remaining elements of $O_s$, i.e. all elements in $O_s\setminus\{p\in O_s:t\not\in p\}$. By construction all $p\in set(\vec{O})$ occurring after the initial sequence $\vec{O}[n]$ are such that $t\in p$. Since $t\simeq^{\textbf{mini}}_{\vec{O}[n]}s$ and $\vec{O}$ is sound, it follows that for all $m\geq n$, $t\simeq^{\textbf{mini}}_{\vec{O}[m]}s$ by Def.~\ref{def:min_update}. Thus, there is no $m$ such that $min_{\preceq^\textbf{mini}_{\vec{O}[m]}}=\{s\}$. Hence, $\preceq$ is not appropriate for learning $S$ via \textbf{mini}.
\end{proof}

The lemma~\ref{lemma:infinity_necessary_mini} ensures that an appropriate plausibility order cannot be deceived by sound and complete data streams such as the one illustrated in Figure~\ref{fig:bad_prior_mini}. 
The problem there is that a distinct state from $s$, namely $u$, is equi-plausible to $s$ and minimal in all observables in which $s$ is minimal. 
Lemma~\ref{lemma:infinity_necessary_mini} excludes such configurations of points in finite epistemic spaces.
% : If such a configuration exists, then one can order  On the other hand, if there are countably many observables that can distinguish $s$ from a state like $u$,  $s$ solves the issue of possibly failing to identify $s$ depending on the order in which observations come in.

\begin{lemma}
    Let $\textbf{S}=\langle S,O\rangle$ be an epistemic space, and let $\preceq$ be a plausibility order over $S$. Consider a state $s$ and assume that there exists a $p\in O_s$ such that $s\in min_{\preceq}p$. Then for any finite sequence $\vec{O}[n]$ of a sound data stream $\vec{O}$ for $s$, it holds that $s\in min_{\preceq^{\textbf{mini}}_{\vec{O[n]}}}p$. 
    \label{lemma:mini_finite_min_persistence}
\end{lemma}

\begin{proof}
    We prove this by induction on $n$. [Base case: $n=0$] Follows directly from the assumption. [Inductive step] Assume that for all $m<n$, $s\in min_{\preceq^{\textbf{mini}}_{\vec{O}[m]}}p$. Consider now the $p'$ occurring in $\vec{O}$ after the initial segment $\vec{O}[m]$. Assume towards a contradiction that $s\not\in min_{\preceq^{\textbf{mini}}_{\vec{O}[m+1]}}p$. Hence there is $t\in p$ such that $t\prec^{\textbf{mini}}_{\vec{O}[m+1]} s$. Since $s\in min_{\preceq^{\textbf{mini}}_{\vec{O}[m]}}p$ by inductive hypothesis, then $t\not\prec^{\textbf{mini}}_{\vec{O}[m]}s$, hence $s\preceq^{\textbf{mini}}_{\vec{O}[m]}t$, since $\preceq$ is total. Since $s\not\in min_{\preceq^{\textbf{mini}}_{\vec{O}[m]}}p'$ (otherwise it would be a minimal element in $p$ after updating on $p'$), there is $t'\prec^{\textbf{mini}}_{\vec{O}[m]}s$, and $t\in p'$. By transitivity, $t'\prec^{\textbf{mini}}_{\vec{O}[m]}t$. Hence, $t\not\in min_{\preceq^{\textbf{mini}}_{O[m]}}p'$. Thus by Def.~\ref{def:min_update}, $s\preceq^{\textbf{mini}}_{\vec{O}[m+1]}t$, which contradicts the assumption that $t\prec^{\textbf{mini}}_{\vec{O}[m+1]}s$. 
\end{proof}

%We now proceed to the characterisation of the classes of appropriate preorders for minimal revision. 
To characterise the plausibility orders appropriate for minimal revision, we make use of the concept of a \textit{finite tell-tale map}. We have already seen a concept of this kind in Definition \ref{def:dfttm}. The concept was first introduced by Dana Angluin in her seminal work \cite{ANGLUIN1980117}. Broadly speaking, a finite tell-tale map 
%is a map that 
assigns to each possible state a special finite set of observables (a finite tell-tale), which guides learning and guarantees that the state can be identified. 
%In order to characterise the class of appropriate plausibility orders for minimal revision, we 
We consider a specific variant of tell-tale maps that satisfies a number of conditions regarding not only the available observables but also the plausibility order of the states.

\begin{definition}
Let $\textbf{S}=\langle S, O\rangle$ be a finite epistemic space and $\preceq$ be a plausibility order over $S$ be given. A \textbf{finite space mini tell-tale map} for $\preceq$ and $\textbf{S}$ is a total map assigning to each state $s\in S$ a \textbf{finite space mini tell-tale set} $F_s\subseteq O$ such that: (i) $F_s\subseteq O_s$; (ii) there is a $p\in F_s$ such that $s\in min_{\preceq}p$ and for all $t\not= s$, $t\simeq s$ such that $t\in min_{\preceq}p $, there exists $q\in F_s$ such that $s\in min_{\preceq} q$ and $t\not\in min_{\preceq} q$ .
\label{def:mini_tell-tale_finite}
\end{definition}

\begin{theorem}
     Let $\preceq$ be a plausibility order on a finite epistemic space $\textbf{S}=\langle S, O\rangle$. The plausibility order $\preceq$ is appropriate to learn $\textbf{S}$ via \textbf{mini} iff there exists a finite space mini tell-tale map for $S$ and $\preceq$. \label{theorem:general_characterisation_mini}
\end{theorem}
\begin{proof}

    [$\Rightarrow$] Let $\textbf{S}=\langle S, O\rangle$ be finite, and let $\preceq$ be %a plausibility order over $S$ 
  appropriate for learning $\textbf{S}$ via \textbf{mini}. 
     %Consider the set $O_s$. 
     Since $O_s\subseteq O_s$, it trivially satisfies (i) from Def.~\ref{def:mini_tell-tale_finite}.
     Assume towards a contradiction that for all $p\in O_s$, either $p\not\in min_{\preceq}p$ or there exists a $t\not=s$ such that $t\in min_{\preceq}p$ and there is no $q\in O_s$ such that $s\in min_{\preceq}q$ and $t\not\in min_{\preceq}q$. By Lemma~\ref{lemma:atleastonemin} we know that there is a $p\in O_s$ such that $s\in min_{\preceq}p$. By assumption, it follows that there exists a $t\in min_{\preceq}p$ such that there is no $q\not=p\in O_s$ such that $s\in min_{\preceq}q$ and $t\not\in min_{\preceq}q$. By Lemma~\ref{lemma:infinity_necessary_mini}, it follows that the set $|\{p\in O_s:t\not\in p\}|=|\mathbb{N}|$. Contradiction. 

     [$\Leftarrow$] Let $\textbf{S}=\langle S,O\rangle$ be a finite epistemic space, and let $\preceq$ be given. Assume there exists a finite space mini tell-tale map for $\preceq$ and $S$. Let $s\in S$ be arbitrary and $F_s$ be its corresponding finite space mini tell-tale set. 
 %Since $F_s$ is a finite tell-tale set, then 
 Since $S$ is finite, also $O_s$ is a finite space mini tell-tale set. 
Consider an arbitrary sound and complete stream $\vec{O}$ for $s$. Since $O_s$ is finite and $\vec{O}$ is complete, there exists an $n\in \mathbb{N}$ such that $O_s\subseteq set(\vec{O}[n])$. Since $O_s$ is a finite space mini tell-tale set for $s$, by Def.~\ref{def:mini_tell-tale_finite} there exists a $p\in O_s$ such that $s\in min_{\preceq}p$. %Consider the first occurrence in $\vec{O}$ of a $p\in O_s$ such that $s\in min_{\preceq}p$. 
Let $\vec{O}[m]$ be the sequence where the last element is the first occurrence of a $p\in O_s$ such that $s\in min_{\preceq}p$ in $\vec{O}$. By Lemma~\ref{lemma:mini_finite_min_persistence}, $s\in min_{\preceq^{\textbf{mini}}_{\vec{O}[m-1]}}p$. By Def.~\ref{def:min_update}, $s\in min_{\preceq^{\textbf{mini}}_{\vec{O}[m]}}$, and thus $s\in L^{\textbf{PLAUS}}_{\textbf{mini}}(S,O[m])$. If there is no $t\not=s$ such that $t\in  L^{\textbf{PLAUS}}_{\textbf{mini}}(S,O[m])$, then for all $m'\geq m$, $L^{\textbf{PLAUS}}_{\textbf{mini}}(S,O[m])=\{s\}$ (since $\vec{O}$ is sound with respect to $s$ will forever be the uniquely minimal element for all $p\in O_s$).  If there is a $t\not=s$ such that $t\in L^{\textbf{PLAUS}}_{\textbf{mini}}(S,O[m])$, then $t\simeq^{\textbf{mini}}_{\vec{O}[m]} s$. This implies that $s\simeq t$, and $t\in min_{\preceq}p$. Since by assumption $\textbf{S}$ has a finite space mini tell-tale, there is $q$ such that $s\in min_{\preceq}q$ and $t\not\in min_{\preceq}q$, thus $t\not\in q$. Since $p$ was the first occurrence of a $p\in O_s$ such that $s\in min_{\preceq}p$, since $q\in O_s$ and $s\in min_{\preceq}q$, there is $m'>m$ such that $q$ is the last observable occurring in the initial segment $\vec{O}[m']$ of $\vec{O}$, since $\vec{O}$ is complete. Take the smallest such $s'$. Since $s\in min_{\preceq^{\textbf{mini}}_{\vec{O}[m]}}$, and $\vec{O}$ is sound with respect to $s$, $s\in min_{\preceq^\textbf{mini}_{\vec{O}[m'-1]}}$ and also $s\in min_{\preceq^\textbf{mini}_{\vec{O}[m'-1]}}q$, and since $t\not\in q$, $s\prec^\textbf{mini}_{\vec{O}[m']}t$ by Def.~\ref{def:min_update}. Since $t$ was arbitrary, for all $t\in S$ such that $t\not=s$ and $t\simeq^\textbf{mini}_{\vec{O}[m]}$, there is $k>m$ such that $s\prec^\textbf{mini}_{\vec{O}[k]}t$. Since the space is finite, there can only be finitely many $t\not=s$ such that $t\in min_{\preceq^\textbf{mini}_{\vec{O}[m]}}$, and thus there is $k'$ such that for all such $t\in S$, $s\prec^\textbf{mini}_{\vec{O}[k']}t$, and $min_{\preceq^\textbf{mini}_{\vec{O}[k']}}=\{s\}$. By Def.~\ref{def:min_update} and the soundness of $\vec{O}$, for all $k''\geq k$, $L^{\textbf{PLAUS}}_{\textbf{mini}}(S,O[k''])=\{s\}$.
\end{proof}

% Let us conclude 

\subsection{No Guarantees for Minimal Revision when Learning with Errors}
\label{sec:learning_errors}
Theorem~\ref{theorem:general_characterisation_mini} characterises the appropriate priors when learning via \textbf{mini} under the assumption that no erroneous observation can occur (i.e., the data streams are sound). Does the characterisation still hold when considering learning scenarios in which errors can occur?
Here, we consider fair streams \cite{baltag2019truth}, i.e., data streams that can contain only finitely many errors, all of which will eventually be corrected. 
\begin{definition}[\textrm{\cite[Def.~$27$]{baltag2019truth}}]
    Let $\textbf{S}=\langle S, O\rangle$ be a negation-closed epistemic space. A \textbf{data stream} $\vec{O}$ from $O$ is \textbf{fair with respect to the state} $s$ precisely in case: (i) $\vec{O}$ is complete with respect to $s$; (ii) there is an $n\in \mathbb{N}$ such that for all $k\geq n$, $s\in O_k$; (iii) for every $i\in \mathbb{N}$ such that $s\not\in O_i$, there is a $k>i$ such that $O_k=\bar{O_i}$.\footnote{All definitions in Section~\ref{sec:preliminary} and Definition~\ref{def:appropriateness} can be straightforwardly adapted for the case of fair data-streams.}
    \label{def:fair_streams}
\end{definition}

When considering fair data streams, not only Theorem~\ref{theorem:general_characterisation_mini} is no longer an adequate characterisation of the appropriate plausibility orders to learn via \textbf{mini}, but also \textbf{mini} is not anymore a universal method on the class of negation-closed finite epistemic spaces.\footnote{\textbf{mini}'s sensitivity to errors was highlighted in \cite{Papadamos:2023ur} (Proposition 3.3), where it was applied to cases of cognitive bias, in which some observations are (systematically) misinterpreted by the revising agent.} 
% . We do this by considering a negation-closed finite epistemic space, providing a plausibility order that is appropriate to learn the space via \textbf{mini} on sound and complete data streams, and finally showing that the same plausibility order is not appropriate to learn the space via \textbf{mini} on fair data streams. 

\begin{proposition}
    The following two hold: (i) There exist plausibility orders appropriate for learning some negation-closed space via \textbf{mini} on sound and complete data streams that are not appropriate for learning the same space via \textbf{mini} on fair data streams; (ii) Minimal revision is not universal on fair data-streams on the class of all finite negation-closed epistemic spaces.     \label{prop:not_fair} 
\end{proposition}

\begin{proof} 

[Point (i)] Consider the epistemic space at the center of Fig.~\ref{fig:dfttspaces}. The space is negation-closed and the plausibility order in which for all $s,t\in S$, $t\simeq s$ is appropriate to identify the space via \textbf{mini} on sound and complete data-streams by the proof of Th.~\ref{prop:at_least_DFTT_mini}. Consider the state $s$ and consider the data stream starting with the sequence $q,\bar{p},p$ followed by an infinite stream of $\bar{q}$. This stream is fair and prevents \textbf{mini} from ever identifying $s$, since $w$ will always be a minimal element with $s$ after the initial segment $q,\bar{p},p$. [Point (ii)] Consider again the epistemic space at the center of Fig.~\ref{fig:dfttspaces}. The only appropriate plausibility order to learn that space via \textbf{mini} on sound and complete data streams is such that $s,t\in S, t\simeq s$. This order is not appropriate for learning the space via \textbf{mini} on fair data streams as shown above for (i).
\end{proof}
We leave to future work a characterisation of the plausibility assignments that are appropriate for minimal revision on fair data streams (and of the class of spaces on which they exist).

\section{What Priors are Appropriate for Stronger Methods?}
\label{sec:stronger}

We turn to considering how one ought to assign priors when employing more radical belief revision methods than minimal revision. In particular, we consider \textit{conditioning} and \textit{lexicographic upgrade}. 

\begin{definition}[Conditioning \cite{Baltag2016smetss}]
\label{def:conditioning}
    Let $\textbf{B}=\langle S, O, \preceq \rangle$ be a plausibility space, and let $p\in O$. The one-step revision function $\textbf{cond}(\textbf{B}, p)$ generates the plausibility space $\textbf{B}'=\langle S', O', \preceq' \rangle$, where $S'=S\cap p$, $O'=\{q\in \mathcal{P}(S'): \exists q'\in O \textrm{ such that } q=q'\setminus\bar{p}\}$, and ${\preceq}'= {\preceq_p}$.
\end{definition}
\begin{definition}[Lexicographic Upgrade \cite{Baltag2016smetss}]
Let $\textbf{B}=\langle S, O, \preceq \rangle$ be a plausibility space, and $p\in O$. The one-step revision function $\textbf{lex}(\textbf{B}, p)$ generates the plausibility space $\textbf{B}'=\langle S, O, \preceq' \rangle$, where $s\preceq' t$ if and only if: $s\preceq_p t$, or $s\preceq_{\bar{p}}t$, or $s\in p$ and $t\not\in p$.
\label{def:lex_update}
\end{definition}

When revising with \textbf{cond} upon receiving some observation  $p$, an agent simply eliminates all the worlds that do not satisfy $p$, and restricts the plausibility order to the worlds that do satisfy $p$. On the other hand, lexicographic upgrade does not eliminate worlds, but rather moves all the $p$-worlds in front of the $\neg{p}$-worlds (\cite{Spohn1988, vanBenthem07}). 

For the case of conditioning and lexicographic revision, we give a fully general characterisation of appropriate plausibility orders. In particular, we consider spaces that are not necessarily finite, as well as preorders that are not necessarily total.\footnote{Definition~\ref{def:appropriateness} can be straightforwardly adapted to the case of arbitrary preorders.} As done for \textbf{mini}, we start by defining a tell-tale map variant.

\begin{definition}
Let $\textbf{S}=\langle S, O\rangle$ be an epistemic space and $\preceq$ be a preorder over $S$ be given. A \textbf{generalised conditioning tell-tale map} for $\preceq$ and $\textbf{S}$ is total map assigning to each state $s\in S$ a set $F_s\subseteq O$ such that: (i) $F_s$ is finite; (ii) $F_s\subseteq O_s$; (iii) for all $t\in S$, if $t\sim s$, $F_s\subseteq O_t$ and $s\not=t$, then $s\prec t$; (iv) for all finite $F'_s\supseteq F_s\subseteq O_s$, for all $t\in S$, if $t\not\sim s$ and $F'_s\subseteq O_t$, then 
             there exists a $v\in S$ such that $v\prec t$ and $F'_s\subseteq O_v$.
    \label{def:gen_cond_telltale}
\end{definition}

One can then prove that the existence of a generalised conditioning tell-tale characterises appropriate preorders for conditioning as well as lexicographic upgrade:

\begin{theorem}
Let $\preceq$ be a preorder over an epistemic space $\textbf{S}$. The following are equivalent: (i) There exists a conditioning tell-tale map for $\textbf{S}$ and $\preceq$; (ii) $\preceq$ is appropriate to learn $\textbf{S}$ via \textbf{cond}; (iii) $\preceq$ is appropriate to learn $\textbf{S}$ via \textbf{lex}.
\label{cor:samecondlex}
\end{theorem}

%Let us conclude by considering whether Theorem~\ref{cor:samecondlex} still holds when learning with fair streams (Definition~\ref{def:fair_streams}). 
It is easy to see that the equivalence between $(i)$ and $(ii)$ in Theorem~\ref{cor:samecondlex} fails when considering fair streams (Definition~\ref{def:fair_streams}): Since \textbf{cond} irreparably eliminates the actual state as soon as an error occurs,  
%in any fair data-stream, 
no preorder is appropriate to learn via \textbf{cond} on fair streams. Finally, since any lexicographic upgrade on a finite initial segment of a fair data stream can be shown to be equivalent to the lexicographic upgrade obtained by removing the corrected errors from the sequence \cite{baltag2019truth}, the equivalence between $(i)$ and $(iii)$ in Theorem~\ref{cor:samecondlex} still holds, and thus the existence of a generalised conditioning tell-tale map also characterises the appropriate preorders for learning via \textbf{lex} on fair streams in negation-closed spaces.

\section{Conclusion}
\label{sec:conclusion}

We showed that minimal revision is universal when used to learn finitely identifiable spaces (Theorem~\ref{prop:at_least_DFTT_mini}), and that it is consequently universal on finite strongly separated spaces (Theorem~\ref{th:separated_mini}). Thus, it can always be used to learn from positive and negative data on finite spaces (Corollary~\ref{th:negation_closed}) unless erroneous observations possibly occur (Proposition~\ref{prop:not_fair}). In addition, we characterised the prior plausibility orders that are appropriate for learning finite spaces via minimal revision (Theorem~\ref{theorem:general_characterisation_mini}) as well as the preorders appropriate for conditioning and lexicographic upgrade (Theorem~\ref{cor:samecondlex}). 
%Finally, we discussed how, unlike with minimal revision and conditioning, the class of appropriate preorders for lexicographic upgrade is also appropriate when learning with erroneous observations.
% We have considered the problem of characterising the prior plausibility orders that guarantee (when possible) identification in the limit for conditioning (Theorem~\ref{prop:general_characterisation_conditioning}), lexicographic upgrade (Theorem~\ref{prop:general_characterisation_lex}) and minimal upgrade on finite spaces (Theorem~\ref{theorem:general_characterisation_mini}).We have also provided a generalisation of our results for conditioning and lexicographic upgrade to arbitrary preorders (Theorems~\ref{th:general_preorder_cond_lex},\ref{th:general_preorder_lex}), and discussed how all our characterisation theorems still hold in learning scenarios with positive and negative information. In addition, we have shown that the class of preorders that is appropriate for learning via lexicographic upgrade with sound and complete data is the same class that is appropriate for learning via lexicographic upgrade on fair data streams.

Beyond the scope of this paper, several questions are left for future research. First, concerning minimal revision: What is the class of epistemic spaces on which the method is universal? What priors are appropriate for learning via minimal revision on arbitrary spaces, and possibly with wrong information? Second, beyond minimal revision, what priors are appropriate for other revision methods? Finally, we plan to consider appropriate priors from the viewpoint of doxastic attitudes, in the spirit of \cite{degremont2011finite}: What type of doxastic attitudes correspond to holding an appropriate prior?
% Finally, we would like to consider the universality of a method and the appropriateness of priors from a graded standpoint: Even if a method is not universal or a plausibility assignment not appropriate, how likely are the method or the plausibility assignment to lead one to failure? 
\newpage
\section*{Acknowledgements}

Rineke Verbrugge acknowledges support from the project ``Hybrid Intelligence: Augmenting Human Intellect'', a 10-year Gravitation programme funded by the Dutch Ministry of Education, Culture and Science through the Netherlands Organisation for Scientific Research (grant number 024.004.022). 
Zoé Christoff acknowledges support from the project ``Democracy on Social Networks'' (VENI project number Vl.Veni.201F.032) financed by the Netherlands Organisation for Scientific Research (NWO). Edoardo Baccini acknowledges support from the Evert Willem Beth Foundation of the Royal Netherlands Academy of Arts and Sciences (KNAW) (grant number KNAW WF/743 – 13). Edoardo Baccini would also like to thank the ``Reasoning, Rationality and Science'' research group at the Ruhr-University Bochum for insightful discussion on an earlier version of this manuscript. All authors thank three anonymous referees for helpful comments.

\bibliographystyle{eptcs}
\bibliography{bibliography}
\newpage
\section*{Appendix}
\label{sec:Appendix}

\subsection*{Proof of Proposition~\ref{th:finite_id}}
\paragraph{Proposition~\ref{th:finite_id}.}\hspace{-0.2cm}\textit{An epistemic space $\textbf{S}$ is finitely identifiable if and only if there exists a definite finite tell-tale map for $\textbf{S}$.}
\vspace{-0.15cm}
\begin{proof}

    [$\Rightarrow$] Consider an epistemic space $\textbf{S}=\langle S, O\rangle$ and assume that it is finitely identifiable. Therefore there exists a learner $L$ that finitely identifies every state $s\in S$ on any sound and complete data stream for $s$. Consider an arbitrary state $s\in S$ and consider an arbitrary sound and complete data stream $\vec{O}$ for $s$. By Def.~\ref{def:fin_id}, it follows that there exists an $n\in\mathbb{N}$ such that $L(S, \vec{O}[n])=\{s\}$. Define $D_s:=set(\vec{O}[n])$. Assume towards a contradiction that $D_s$ is not a definite finite tell-tale set for $s$. It follows that there exists a $t\in S$ such that $t\not=s$ and $D_s\subseteq O_t$. Consider now a data stream for $\vec{O}_t$ for $t$ such that $\vec{O}_t[n]=\vec{O}[n]$. It follows that $L(\textbf{S}, \vec{O}_t[n])=\{s\}$. Thus, $L$ does not finitely identify $t$ contrary to our initial assumption. 
    
    [$\Leftarrow$] Consider an epistemic space $\textbf{S}=\langle S, O\rangle$ and assume there exists a definite finite tell-tale map for $\textbf{S}$. Let an enumeration $s_0, s_1, ..., s_i, ...$ of the elements of $S$ be given. Define $L$ in the following way: $L(\textbf{S}, \vec{O}[k])=\{s_i\}$, where $i\leq k$ is the smallest $i$ such that $D_{s_i}\subseteq set(O[k])$, if such $s_i$ exists and no guess has been made in the past; else, $L$ makes no guess. Consider an arbitrary state $s_j$ and a sound and complete data stream $\vec{O}$ for $s_j$. Since $\vec{O}$ is complete and $D_{s_j}$ is finite, there exists an $n\in \mathbb{N}$ such that $D_{s_j}\subseteq set(\vec{O}[n])$. Consider the smallest $n'\leq n$ for which $D_{s_j}\subseteq set(\vec{O}[n])$ and $j\leq n'$. Such an $n'$ exists. Consider indeed the least $n$ for which $D_{s_j}\subseteq set(\vec{O}[n])$: If $n\geq j$, then $n':=n$; If $n< j$, set $n':=j$ since $D_{s_j}\subseteq set(\vec{O}[n])\subseteq set(\vec{O}[j])$. We now need to show that $L$ has not made guesses for any $k<n'$. Assume towards a contradiction that there is $k< n'$, such that $L(\textbf{S},\vec{O}[k])=\{s_{j'}\}$. Since $n'$ is the smallest $n\in \mathbb{N}$ for which $D_{s_j}\subseteq \vec{O}[n]$ and $j\leq n'$, it follows that $D_{s_j}\not\subseteq set(\vec{O}[k])$. Thus $s_{j'}\not=s_j$. By definition of $L$, it follows that $D_{s_{j'}}\subseteq set(\vec{O})[k]\subseteq O_{s_j}$. But then $D_{s_{j'}}$ is not a definite finite tell-tale for $s_{j'}$, which contradicts our assumption. Since no previous guess has been made by $L$ before observing the sequence $\vec{O}[n']$, $L(\textbf{S}, \vec{O}[n'])=\{s_j\}$ (if $L(\textbf{S}, \vec{O}[n'])\not=\{s_j\}$, then there would be $s_i\not=s_j$ with $i<j$ and such that $D_{s_i}\subseteq set(\vec{O}[n'])\subseteq O_{s_j}$, and thus $D_{s_i}$ would not be a definite finite tell-tale set for $s_i$ contrary to our assumptions). By definition of $L$, no later guess will be output by $L$. Thus, $L$ finitely identifies $s_j$. Since $s_j$ was arbitrary, $L$ finitely identifies $S$. \end{proof}

\subsection*{\hypertarget{sec:lex_proof}{Proof of Theorem~\ref{cor:samecondlex}}}
\label{sec:cond_proof}

\paragraph{Theorem~\ref{cor:samecondlex}.}\hspace{-0.2cm}\textit{Let $\textbf{S}$ be an epistemic space and let $\preceq$ be a preorder over $\textbf{S}$. The following are equivalent: (i) There exists a conditioning tell-tale map for $\textbf{S}$ and $\preceq$; (ii) $\preceq$ is appropriate to learn $\textbf{S}$ via \textbf{cond}; (iii) $\preceq$ is appropriate to learn $\textbf{S}$ via \textbf{lex}.}

     We prove Theorem~\ref{cor:samecondlex} by proving the following two propositions: 

     \begin{proposition}Let $\preceq$ be a preorder on an epistemic space $\textbf{S}=\langle S, O\rangle$. The preorder $\preceq$ is appropriate to learn $\textbf{S}$ via \textbf{cond}
     if and only if there exists a generalised conditioning tell-tale map for $S$ and $\preceq$.
     \label{prop:cond}
\end{proposition}

\begin{proposition}Let $\preceq$ be a preorder on an epistemic space $\textbf{S}=\langle S, O\rangle$. The preorder $\preceq$ is appropriate to learn $\textbf{S}$ via \textbf{lex}
     if and only if there exists a generalised conditioning tell-tale map for $S$ and $\preceq$.
     \label{prop:lex}
\end{proposition}

Proposition~\ref{prop:cond} proves that $(i)$ and $(ii)$ in Theorem~\ref{cor:samecondlex} are equivalent. Proposition~\ref{prop:lex} proves that $(i)$ and $(iii)$ in Theorem~\ref{cor:samecondlex} are equivalent.

To prove both propositions, we rely on the following definition and lemma.

     \begin{definition}[Locking sequence \textrm{\cite[Def.~$22$]{baltag2019truth}}, originally in \cite{BLUM1975125}]
Let an epistemic space $\textbf{S}=\langle S,O\rangle$, a possible world $s\in S$, a learning method $L$, and a data sequence $\sigma$ be given. Sequence $\sigma$ is called a locking sequence for $s$ if $\sigma$ is sound with respect to $s$ and for each data sequence $\tau$ with $s\in \bigcap set(\tau)$, $L(\textbf{S}, \sigma*\tau)=L(\textbf{S}, \sigma)=\{s\}$.
    \label{def:belief_revision}
\end{definition}

\begin{lemma}[\textrm{\cite[Lemma~$1$]{baltag2019truth}}] If learning method $L$ identifies possible world $s$ in the limit, then there exists a locking sequence for $s$ and $L$.
    \label{lemma:lockingsequence}
\end{lemma}
     
     \subsubsection*{Proof of Proposition~\ref{prop:cond}}
\paragraph{Proposition~\ref{prop:cond}.}\hspace{-0.2cm}\textit{Let $\preceq$ be a preorder on an epistemic space $\textbf{S}=\langle S, O\rangle$. The preorder $\preceq$ is appropriate to learn $\textbf{S}$ via \textbf{cond}
     if and only if there exists a generalised conditioning tell-tale map for $S$ and $\preceq$.}
\begin{proof}
   
[$\Rightarrow$]  Let a preorder $\preceq$ and an epistemic space $\textbf{S}$ be given. Assume that the preorder $\preceq$ is appropriate to learn $\textbf{S}$ via $\preceq$ and \textbf{cond}. Let \textbf{PLAUS} be the plausibility assignment assigning the preorder $\preceq$ to $\textbf{S}$. By Def.~\ref{def:appropriateness}, the learner $L^{\textbf{PLAUS}}_{\textbf{cond}}$ identifies the space $S$ in the limit. 
     By Lemma~\ref{lemma:lockingsequence}, for each state $s\in S$, there exists a locking sequence $\sigma_s$ for $L^{\textbf{PLAUS}}_{\textbf{cond}}$ and $s$.
     
     Consider an arbitrary $s$ and define $F_s:=set(\sigma_s)$. 
      $F_s$ is finite, since $\sigma_s$ is finite, and thus satisfies point $(i)$ in Def.~\ref{def:gen_cond_telltale}. Furthermore, $F_s\subseteq O_s$ since $set(\sigma_s)\subseteq O_s$, and thus satisfies point $(ii)$ in Def.~\ref{def:gen_cond_telltale}.

         We now prove that $F_s$ satisfies $(iii)$ of Def.~\ref{def:gen_cond_telltale}. Assume for contradiction that there is a $t\in S$ such that
         $t\sim s$, 
         $F_s\subseteq O_t$ and $t\not=s$, and $s\not\prec t$. Then, $t\simeq s$, or $t\prec s$. 
         [Case: $t\prec s$] Consider now a sound and complete data stream $\vec{O}$ such that there exists $n$ such that $\vec{O}[n]=\sigma_s$. By Def.~\ref{def:conditioning}, and the fact that $F_s\subseteq O_t$, $t$ is in the updated space after $\vec{O}[n]$, just like $s$. Consider the preorder $\preceq^{\textbf{cond}}_{\vec{O}[n]}$, i.e., the preorder in the updated space $\textbf{cond}(\textbf{PLAUS}(\textbf{S}), \vec{O}[n])$. Since $t\prec s$, we have that $t\prec^{\textbf{cond}}_{\vec{O}[n]}s$. Thus $s\not\in L^{\textbf{PLAUS}}_{\textbf{cond}}(\textbf{S},\vec{O}[n])$, since $s\not\in min_{\preceq^{\textbf{cond}}_{{\vec{O}[n]}}}$. But by assumption, $\sigma_s$ is a locking sequence for $s$ and $L^{\textbf{PLAUS}}_{\textbf{cond}}$, and thus $L^{\textbf{PLAUS}}_{\textbf{cond}}(\textbf{S},\vec{O}[n])=\{s\}$. 
         [Case: $t\simeq s$] Consider now a sound and complete data stream $\vec{O}$ such that there exists $n$ such that $\vec{O}[n]=\sigma_s$. By Def.~\ref{def:conditioning}, and the fact that $F_s\subseteq O_t$, $t$ is in the updated space after $\vec{O}[n]$, just like $s$. 
         % Let us denote with $\preceq'$ is the plausibility order in the updated space $\textbf{cond}(\textbf{PLAUS}(\textbf{S}), \vec{O}[n])$. 
         Since $t\simeq s$, we have that $t\simeq^{\textbf{cond}}_{\vec{O}[n]}s$. 
         By assumption, $\sigma_s$ is a locking sequence for $s$ and $L^{\textbf{PLAUS}}_{\textbf{cond}}$, and thus $L^{\textbf{PLAUS}}_{\textbf{cond}}(\textbf{S},\vec{O}[n])=\{s\}$, and thus $s\in min_{\preceq^{\textbf{cond}}_{\vec{O}[n]}}$. But since $t\simeq^\textbf{cond}_{\vec{O}[n]}s$, then $t\in min_{\preceq^\textbf{cond}_{\vec{O}[n]}}$. Hence $L^{\textbf{PLAUS}}_{\textbf{cond}}(\textbf{S},\vec{O}[n])\not=\{s\}$.

        We now prove that $F_s$ satisfies $(iv)$ of Def.~\ref{def:gen_cond_telltale}. Assume by contradiction that there exists a set $F'_s\supseteq F_s\subseteq O_s$ such that there exists a $t\in S$ such that $t\not\sim s$, $F'_s\subseteq O_t$, and there is no $v$ such that $v\prec r$ and $F'_s\subseteq O_v$. Consider now a sound and complete data-stream $\vec{O}$ such that for some $n\in N$, $\vec{O}[n]=\sigma_s$, and for some $k\geq n$, $set(\vec{O}[k])=F'_s$. Since $\vec{O}[n]$ is a locking sequence for $s$, $L^{\textbf{PLAUS}}_{\textbf{cond}}(\textbf{S},\vec{O}[n])=\{s\}$ and for all sequences $\tau$ for $s$, $L^{\textbf{PLAUS}}_{\textbf{cond}}(\textbf{S},\vec{O}[n]*\tau)=\{s\}$. Hence also $L^{\textbf{PLAUS}}_{\textbf{cond}}(\textbf{S},\vec{O}[k])=\{s\}$.
         By assumption, we know that there is a $t\in S$ such that $t\not\sim s$ and $set(\vec{O}[k])\subseteq O_t$. Therefore, such a $t$ is in the updated model $\textbf{cond}(\textbf{PLAUS}(\textbf{S}), \vec{O}[k])$, and moreover $t\not\sim^{\textbf{cond}}_{\vec{O}[k]} s$. By assumption, for such a $t$, there is no $v$ such that $v\prec r$ and $F'_s\subseteq O_v$. 
         Hence, for all $v \in S$, either $v\not\prec t$ or $F'_s\not\subseteq O_v$. Consider such a $v$. If $F'_s\not\subseteq O_v$, then $v$ has been eliminated. Therefore, for all $v\sim t$, either $v$ has been eliminated or $t$ is at least as good as $v$ in the updated model, i.e., $t\preceq^{\textbf{cond}}_{\vec{O}[k]}v$. Therefore, if $v\sim^{\textbf{cond}}_{\vec{O}[k]} t$, then $t\preceq^{\textbf{cond}}_{{\vec{O}[k]}} v$. Thus, in the model updated on the sequence $\vec{O}[k]$, we have $t\in min_{\preceq^{\textbf{cond}}_{\vec{O}[k]}}$, and thus $t\in L^{\textbf{PLAUS}}_{\textbf{cond}}(\textbf{S},\vec{O}[k])$. Contradiction.

[$\Leftarrow$]  Let a preorder $\preceq$ and an epistemic space $\textbf{S}$ be given such that for each $s$ there exists a set $F_s$ as in Def.~\ref{def:gen_cond_telltale}. Let the plausibility assignment \textbf{PLAUS} be given that assigns the preorder $\preceq$ to the space $\textbf{S}$.
     Consider now an arbitrary world $s$, and a sound and complete data stream $\vec{O}$ for $s$. Since $F_s$ is finite, there exists an $n\in\mathbb{N}$ such that $F_s\subseteq set(\vec{O}[n])$. 
     We now show that $L^{\textbf{PLAUS}}_{\textbf{cond}}$ identifies $s$ in the limit, by showing that $s$ is the unique minimal element in $\textbf{cond}(\textbf{PLAUS}(\textbf{S}),\vec{O}[n])$, and that it is the unique minimal element for any $k>n$.
    Consider $t\in S$ distinct from $s$ in the model obtained after updating on $\vec{O}[n]$. By definition of \textbf{cond}, it must be the case that $set(\vec{O}[n])\subseteq O_t$ and thus $F_s\subseteq O_t$.    Consider the preorder $\preceq^{\textbf{cond}}_{\vec{O}[n]}$ in the updated space $\textbf{cond}(\textbf{PLAUS}(\textbf{S}),\vec{O}[n])$. We consider two cases.
    
    [Case 1: $t\sim^{\textbf{cond}}_{\vec{O}[n]} s$] It follows that $t\sim s$. Since $F_s\subseteq set(O[n])$, and since $set(O[n]) \subseteq O_t$ (since otherwise $t$ would have been eliminated), $F_s\subseteq O_t$. By Def.~\ref{def:gen_cond_telltale} it follows that $s\prec t$, and thus $s\prec^{\textbf{cond}}_{\vec{O}[n]} t$. Since $t$ is arbitrary, it follows that for all $t\sim^{\textbf{cond}}_{\vec{O}[n]} s$, $s\prec^{\textbf{cond}}_{\vec{O}[n]}t$, and thus $s\in min_{\preceq^{\textbf{cond}}_{\vec{O}[n]}}$.

        [Case 2: $t\not\sim^{\textbf{cond}}_{\vec{O}[n]} s$] Consider a $t\not\sim^{\textbf{cond}}_{\vec{O}[n]} s$. Since $t\not\sim^{\textbf{cond}}_{\vec{O}[n]} s$, and since $set(\vec{O}[n])\subseteq O_t$, we have that $F_s \subseteq set(\vec{O}[n])$. By Def.~\ref{def:gen_cond_telltale} there exists a $v\in S$ such that $v\prec t$ and $set(\vec{O}[n])\subseteq O_v$. It follows that $v\prec^{\textbf{cond}}_{\vec{O}[n]} t$, and thus $t\not\in min_{\preceq^{\textbf{cond}}_{\vec{O}[n]}}$. Since $t$ was arbitrary, there is no $\in S$ such that $t\not\sim^{\textbf{cond}}_{\vec{O}[n]}s$ and $t\in min_{\preceq^{\textbf{cond}}_{\vec{O}[n]}}$.
     
     It follows that, $min_{\preceq^{\textbf{cond}}_{\vec{O}[n]}}\textbf{cond}(\textbf{PLAUS}(\textbf{S}),\vec{O}[n])=\{s\}$. Thus $L^\textbf{PLAUS}_{\textbf{cond}}(\textbf{S}, \vec{O}[n])=\{s\}$.

     We now show that no new minimal element is created. Consider an arbitrary $k>n$, and assume by contradiction that some $t\not=s$ is a minimal element after updating on the sequence $\vec{O}[k]$. Note that $s\in min_{\preceq^{\textbf{cond}}_{\vec{O}[k]}}$, since $\vec{O}$ is sound with respect to $s$, and \textbf{cond} does not change the relative order of the worlds. Also, $F_s\subseteq set(\vec{O}[n])\subseteq set(\vec{O}[k])$. Since $set(\vec{O}[k])\subseteq O_t$, we have $F_s
     \subseteq O_t$. Either $t\sim^{\textbf{cond}}_{\vec{O}[k]} s$ or $t\not\sim^{\textbf{cond}}_{\vec{O}[k]} s$. In the first case, $t\sim s$. Thus, by assumption, $s\prec t$, and by definition of conditioning $s\prec^{\textbf{cond}}_{\vec{O}[k]} t$, which contradicts the assumption that $t$ is minimal. 
     Therefore, it must be the case that $t\not\sim^{\textbf{cond}}_{\vec{O}[k]} s$. This implies that $t\not\sim s$. But since $F_s\subseteq set(\vec{O}[k])$, there exists a $v\in S$ s.t. $v\prec t$ and $set(\vec{O}[k])\subseteq O_v$. It follows that, $v\prec^{\textbf{cond}}_{\vec{O}[k]} t$ and therefore $t$ is not a minimal element of $\textbf{cond}(\textbf{PLAUS}(\textbf{S}),\vec{O}[k])$.
     % by assumption , we know that there exist infinitely many strictly better world than $t$, for which there exist infinitely many better world and so on, which are all consistent with the information $set(\vec{O}[k])$ since $set(\vec{O}[k])\supseteq F_s$. All these world will be in the model $\textbf{cond}(\textbf{PLAUS}(S), \vec{O}[k])$, and thus $t\not\in min_{\preceq'}\textbf{cond}(\textbf{PLAUS}(S), \vec{O}[k])$. 
     Contradiction. 
\end{proof}

\subsubsection*{\hypertarget{sec:lex_proof}{Proof of Proposition~\ref{prop:lex}}}
\label{sec:lex_proof}
\paragraph{Proposition~\ref{prop:lex}.}\hspace{-0.2cm}\textit{Let $\preceq$ be a preorder on an epistemic space $\textbf{S}=\langle S, O\rangle$. The preorder $\preceq$ is appropriate to learn $\textbf{S}$ via \textbf{lex}
     if and only if there exists a general conditioning tell-tale map for $S$ and $\preceq$.}

To prove the proposition above we need a number of facts about \textbf{lex}.

\begin{fact}
    Let $\preceq$ be a  preorder on an epistemic space $\textbf{S}=\langle S, O\rangle$. If $s\sim t$, then $s\sim^\textbf{lex}_{p} t$, for any $p\in O$.
    %, where $\preceq'$ is the preorder in the updated space $\textbf{Lex}(\textbf{PLAUS}(\textbf{S}), p)$,
    \label{fact:pers_comparability}
\end{fact}

\begin{proof}
    Assume that $s\sim t$. Then $s\preceq t$ or $t\preceq s$. Without loss of generality, assume that $s\preceq t$. There are four cases. 
    [Case 1: $s\in p$ and $t\in p$] Since $s\in p$, $t\in p$ and $s\preceq t$ by assumption, we have that $s\preceq^\textbf{lex}_{p} t$ by Def.~\ref{def:lex_update}. Hence $s\sim^\textbf{lex}_{p} t$.
    [Case 2: $s\not\in p$ and $t\not\in p$] Since $s\not\in p$, $t\not\in p$ and $s\preceq t$ by assumption, we have that $s\preceq^\textbf{lex}_{p} t$ by Def.~\ref{def:lex_update}. Hence $s\sim^\textbf{lex}_{p} t$. 
    [Case 3: $s\in p$ and $t\not\in p$] If $s\in p$ and $t\not\in p$, then $s\preceq^\textbf{lex}_{p} t$ by Def.~\ref{def:lex_update}. Hence, $s\sim^\textbf{lex}_{p} t$. 
    [Case 4: $s\not\in p$ and $t\in p$] If $s\not\in p$ and $t\in p$, then $s\preceq^\textbf{lex}_{p} t$ by Def.~\ref{def:lex_update}. Hence, $s\sim^\textbf{lex}_{p} t$. 
\end{proof}

\begin{fact}
    Let $\preceq$ be a  preorder on an epistemic space $\textbf{S}=\langle S, O\rangle$. Consider $s,t\in S$ such that $s\not\sim t$, and consider a sound data stream $\vec{O}$ for $s$. Then, for any $n\in \mathbb{N}$, if $set(\vec{O}[n])\subseteq O_t$, then $s\not\sim^\textbf{lex}_{\vec{O}[n]} t$.
    % , where $\preceq'$ is the preorder in the space $\textbf{lex}(\textbf{PLAUS}(\textbf{S}), \vec{O}[n])$.
    \label{fact:persistence_incomp}
\end{fact}

\begin{proof}
    By induction on $n$.   
    [Base Case: $n=0$] Given that $\vec{O}[0]$ is empty, the updated space is the same as the initial space, thus $t\not\sim s$. 
    [Inductive Step] Assume that for all $n< m$, $t\not\sim^\textbf{lex}_{\vec{O}[n]} s$ in the corresponding update $\textbf{lex}(\textbf{PLAUS}(\textbf{S}), \vec{O}[n])$. Therefore, $t\not\sim^\textbf{lex}_{\vec{O}[m-1]}s$. Hence, $t\not\preceq^\textbf{lex}_{\vec{O}[m-1]}s$ and $s\not\preceq^\textbf{lex}_{\vec{O}[m-1]}t$. Assume $set(\vec{O}[m])\subseteq O_t$, then $t,s\in p$, where $p\in O_s$ is the proposition occurring in $\vec{O}$ after the sequence $\vec{O}[m-1]$. By Def.~\ref{def:lex_update}, since $s,t\in p$, we have that $s\preceq^\textbf{lex}_{\vec{O}[m]} t$ if and only if $s\preceq^\textbf{lex}_{\vec{O}[m-1]} t$, and $t\preceq^\textbf{lex}_{\vec{O}[m]} s$ if and only if $t\preceq^\textbf{lex}_{\vec{O}[m-1]} s$. Using the inductive hypothesis, we then conclude that $s\not\preceq^\textbf{lex}_{\vec{O}[m]} t$ and $t\not\preceq^\textbf{lex}_{\vec{O}[m]} s$. Hence, $s\not\sim^\textbf{lex}_{\vec{O}[m]} t$.
\end{proof}

\begin{fact}
    Let $\preceq$ be a  preorder on an epistemic space $\textbf{S}=\langle S, O\rangle$. Consider $s,t\in S$ such that $s\prec t$, and consider a sound data stream $\vec{O}$ for $s$. We have that for any $n\in \mathbb{N}$, $s\prec^\textbf{lex}_{\vec{O}[n]} t$.
    % , where $\preceq'$ is the preorder in the space $\textbf{lex}(\textbf{PLAUS}(\textbf{S}), \vec{O}[n])$.
    \label{fact:persistence_strict}
\end{fact}

\begin{proof}
    Proof by induction on $n$. 
    [Base Case: $n=0$] Given that $\vec{O}[0]$ is empty, the updated space is the same as the initial space, thus $s\prec t$. 
    [Inductive Step] Assume that for all $n< m$, $s\prec^\textbf{lex}_{{\vec{O}[n]}} t$ in the corresponding update $\textbf{lex}(\textbf{PLAUS}(\textbf{S}), \vec{O}[n])$. Therefore, $s\prec^\textbf{lex}_{{\vec{O}[m-1]}}t$. Since the stream is sound with respect to $s$, we have $s\in p$, where $p\in O$ occurs in the stream $\vec{O}$ after the sequence $\vec{O}[m-1]$. Either $t\in p$, or $t\not\in p$. If $t\in p$, then $s\preceq^\textbf{lex}_{{\vec{O}[m]}} t$ and $t\not\preceq^\textbf{lex}_{{\vec{O}[m]}} s$ by Def.~\ref{def:lex_update}; hence, $s\prec^\textbf{lex}_{{\vec{O}[m]}} t$. If $t\not\in p$, then again $s\preceq^\textbf{lex}_{{\vec{O}[m]}} t$ and $t\not\preceq^\textbf{lex}_{{\vec{O}[m]}} s$ by Def.~\ref{def:lex_update}, and thus $s\prec^\textbf{lex}_{{\vec{O}[m]}} t$.
\end{proof}

\begin{fact}
     Let $\preceq$ be a  preorder on an epistemic space $\textbf{S}=\langle S, O\rangle$. If there is an $s\in S$ such that $s\in min_\preceq$, then for any sound data stream $\vec{O}$ for $s$, for any $n\in \mathbb{N}$, $s\in min_{\preceq^\textbf{lex}_{{\vec{O}[n]}}}\textbf{lex}(\textbf{PLAUS}(\textbf{S}), \vec{O}[n])$.
    \label{fact:min_preservation_sound}
\end{fact}

\begin{proof}
    By induction on $n$. 
    [Base case] Follows from the fact that $s\in min_\preceq$. 
    [Inductive step] Assume that for all $n< m$, $s\in min_{\preceq^\textbf{lex}_{{\vec{O}[n]}}}$, where $\preceq^\textbf{lex}_{{\vec{O}[n]}}$ is the preorder in the 
 space $\textbf{lex}(\textbf{PLAUS}(\textbf{S}), \vec{O}[n])$. This is also true for $n=m-1$. Therefore, $s\in min_{\preceq^\textbf{lex}_{{\vec{O}[m-1]}}}$; hence, for all $t\sim^\textbf{lex}_{{\vec{O}[m-1]}} s$, $s\preceq^\textbf{lex}_{{\vec{O}[m-1]}}t$. 
 Now consider $t\sim^\textbf{lex}_{{\vec{O}[m]}}s$. There are two possibilities: either $t\sim^\textbf{lex}_{{\vec{O}[m-1]}}s$ or not, i.e., $t\not\sim^\textbf{lex}_{{\vec{O}[m-1]}}s$. 
 [Case 1: $t\sim^\textbf{lex}_{{\vec{O}[m-1]}}s$]  Since $\vec{O}$ is sound with respect to $s$, we have $s\in p$, where $p\in O$ is the observable occurring in $\vec{O}$ after the initial segment $\vec{O}[m-1]$. Either $t\not\in p$ or $t\in p$. If $t\not\in p$, then $s\preceq^\textbf{lex}_{{\vec{O}[m]}} t$ by Def.~\ref{def:lex_update}; if $t\in p$, then since $s\preceq^\textbf{lex}_{{\vec{O}[m-1]}} t$ by inductive hypothesis, we have $s\preceq^\textbf{lex}_{{\vec{O}[m]}} t$ by Def.~\ref{def:lex_update}.
 [Case 2: $t\not\sim^\textbf{lex}_{{\vec{O}[m-1]}}s$] Since $\vec{O}$ is sound with respect to $s$, we have $s\in p$, where $p\in O$ is the observable occurring in $\vec{O}$ after the initial segment $\vec{O}[m-1]$. Either $t\not\in p$ or $t\in p$. If $t\not\in p$, then $s\preceq^\textbf{lex}_{{\vec{O}[m]}}t$ by Def.~\ref{def:lex_update}; if  $t\in p$, then $s\not\sim^\textbf{lex}_{{\vec{O}[m]}}t$ by Def.~\ref{def:lex_update}. 
 Hence, for all $t\in S$, if $t\sim^\textbf{lex}_{{\vec{O}[m]}}s$, then $s\preceq^\textbf{lex}_{{\vec{O}[m]}} t$. Hence $s\in min_{\preceq^\textbf{lex}_{{\vec{O}[m]}}}$.
\end{proof}

We can now show that the existence of a general conditioning tell-tale map (Definition~\ref{def:gen_cond_telltale}) is necessary (Lemma~\ref{lemma:lex_necessary}, right to left direction of Proposition~\ref{prop:lex}) and sufficient (Lemma~\ref{lemma:lex_sufficient}, left to right direction of Proposition~\ref{prop:lex}) for a preorder to be appropriate to learn via \textbf{lex}.

\begin{lemma}
 Let $\preceq$ be a  preorder on an epistemic space $\textbf{S}=\langle S, O\rangle$. If $\preceq$ is appropriate to learn $\textbf{S}$ via \textbf{lex}, then there exists a conditioning tell-tale map for $S$ and $\preceq$.
     \label{lemma:lex_necessary}
 \end{lemma}

 \begin{proof}
    Let $\preceq$ be a  preorder on an epistemic space $\textbf{S}=\langle S, O\rangle$ that is identifiable via lexicographic upgrade. Let the plausibility assignment $\textbf{PLAUS}$ be given that assigns $\preceq$ to $\textbf{S}$.
    Assume that $L^\textbf{PLAUS}_{\textbf{lex}}$ identifies $\textbf{S}$ in the limit. By Lemma~\ref{lemma:lockingsequence}, for every state $s\in S$, there exists a locking sequence $\sigma_s$. Consider an arbitrary state $s$. Let $F_s:=set(\sigma_s)$. We now show that $F_s$ is a conditioning tell-tale for $s$. 
    $F_s$ is finite, since $\sigma_s$ is finite, and thus satisfies point $(i)$ of Def.~\ref{def:gen_cond_telltale}. 
    $F_s\subseteq O_s$, since $\sigma_s$ is sound with respect to $s$. Thus, $F_s$ satisfies point $(ii)$ of Def.~\ref{def:gen_cond_telltale}.    
    We now show that $F_s$ satisfies $(iii)$ in Def.~\ref{def:gen_cond_telltale}. Assume by contradiction that there is some $t\sim s$ such that $F_s\subseteq O_t$, $s\not= t$ and $s\not\prec t$. 
    Since $s\not\prec t$, either $t\prec s$ or $t\simeq s$. 
    [Case 1: $t\prec s$] Consider a sound and complete data stream $\vec{O}$ for which there is an $n\in \mathbb{N}$ such that $\vec{O}[n]=\sigma_s$.  Since $t \prec s$, $F_s\subseteq O_t$ and $\vec{O}[n]$ is sound with respect to $s$, we have that $t\prec^\textbf{lex}_{{\vec{O}[n]}} s$ by Fact~\ref{fact:persistence_strict}.
    Hence, $s\not\in min_{\preceq^\textbf{lex}_{{\vec{O}[n]}}}\textbf{lex}(\textbf{PLAUS}(\textbf{S}), \vec{O}[n])$, and thus $s\not\in L^\textbf{PLAUS}_{\textbf{lex}}(\textbf{S},\vec{O}[n])$, which contradicts the fact that $\sigma_s$ is a locking sequence for $s$. 
    [Case 2: $t\simeq s$] Similar as the case above: $s$ is a minimal element if and only if $t$ is a minimal element. And hence we contradict the assumption that $\sigma_s$ is a locking sequence.
    
    We now show that $F_s$ satisfies $(iv)$ of Def.~\ref{def:gen_cond_telltale}. Assume in order to derive a contradiction that there exists a finite set $F'_s\supseteq F_s\subseteq O_s$ such that some $t\in S$ with $t\not\sim s$, and $F'_s\subseteq O_t$ has no $v$ such that $v\prec t$ and $F'_s\subseteq O_v$. Consider a data stream $\vec{O}$ such that there is an $n\in\mathbb{N}$ such that $\vec{O}[n]=\sigma_s$, and for which there is a $k\geq n$ such that $set(\vec{O}[k])=F'_s$. 
    Since $\vec{O}[n]$ is a locking sequence for $s$, $L^\textbf{PLAUS}_{\textbf{lex}}(\textbf{S},\vec{O}[k])=\{s\}$. Consider now $t\in S$ as above: for all $v\in S$, either $v\not\prec t$ or $F'_s\not\subseteq O_v$. We now show that $t$ is a minimal element of $\preceq^\textbf{lex}_{{\vec{O}[k]}}$ in the updated plausibility space $\textbf{lex}(\textbf{PLAUS}(\textbf{S}), \vec{O}[k])$. 
    Consider $v\sim^\textbf{lex}_{{\vec{O}[k]}} t$: either $v\preceq^\textbf{lex}_{{\vec{O}[k]}} t$ or $t\preceq^\textbf{lex}_{{\vec{O}[k]}} v$. If $v\preceq^\textbf{lex}_{{\vec{O}[k]}} t$, it must be the case that $F'_s\subseteq O_v$, otherwise $v$ would have become strictly less plausible than $t$ at some point and would have stayed so, since $F'_s\subseteq O_t$. But by assumption, it follows that $v\not\prec t$, and hence, either $v\not\sim t$ or, $t\preceq v$. If $v\not\sim t$, then $v\not\sim^\textbf{lex}_{{\vec{O}[k]}}t$, since both $F'_s\subseteq O_t$ and $F'_s\subseteq O_v$. But $v\sim^\textbf{lex}_{{\vec{O}[k]}} t$, and thus contradiction. Therefore, it must be the case that $t\preceq v$. Since both $t,v$ satisfy all the propositions occurring in $\vec{O}[k]$, we have that $t\preceq^\textbf{lex}_{{\vec{O}[k]}} v$. In all cases, for all $v\sim^\textbf{lex}_{{\vec{O}[k]}} t$, we have $t\preceq^\textbf{lex}_{{\vec{O}[k]}} v$, which implies that $t\in min_{\preceq ^\textbf{lex}_{{\vec{O}[k]}}}$, and thus $t\in L^\textbf{PLAUS}_{\textbf{lex}}(\textbf{S},\vec{O}[k])$. But since $\sigma_s$ is a locking sequence and $s\in \bigcap set(\vec{O}[k])$, we may conclude $L^\textbf{PLAUS}_{\textbf{lex}}(\textbf{S},\vec{O}[k])=\{s\}$. Contradiction.    
\end{proof}

\begin{lemma}
 Let $\preceq$ be a  preorder on an epistemic space $\textbf{S}=\langle S, O\rangle$. If there exists a conditioning tell-tale map for $S$ and $\preceq$, then $\preceq$ is appropriate to learn $\textbf{S}$ via \textbf{lex}.
     \label{lemma:lex_sufficient}
 \end{lemma}

 \begin{proof}
     Let a space $\textbf{S}$ be given and let $\textbf{PLAUS}$ be a plausibility assignment assigning a preorder $\preceq$ to $\textbf{S}$ such that there exists a conditioning tell-tale map for $S$ and $\preceq$.

Consider a state $s\in S$ and a sound and complete data stream $\vec{O}$ for $s$. Consider a set $F_s$ for $s$ as in Def.~\ref{def:gen_cond_telltale}. Since the data stream is complete and $F_s$ is finite, there exists an $n\in\mathbb{N}$ such that $F_s\subseteq set(\vec{O}[n])$.

We now show that $L_{\textbf{lex}}^{\textbf{PLAUS}}$ outputs $s$ after updating on $\vec{O}[n]$, and forever after that. Consider a $t\not=s$ in $\textbf{lex}(\textbf{PLAUS}(\textbf{S}), \vec{O}[n])$. Either $t\sim^\textbf{lex}_{{\vec{O}[n]}}s$ or $t\not\sim^\textbf{lex}_{{\vec{O}[n]}} s$.  

[Case 1: $t\sim^\textbf{lex}_{{\vec{O}[n]}}s$] Either $t\not\sim s$ or $t\sim s$. If $t\not\sim s$, then by Fact~\ref{fact:persistence_incomp}, we have $set(\vec{O}[n])\not\subseteq O_t$, and thus there exists an $m\leq n$ such that $s\preceq^\textbf{lex}_{{\vec{O}[m]}} t$, and $t\not\preceq^\textbf{lex}_{{\vec{O}[m]}} s$ by Def.~\ref{def:lex_update}. Since $\vec{O}$ is sound with respect to $s$, we have that $s\prec^\textbf{lex}_{{\vec{O}[n]}} t$ by Fact~\ref{fact:persistence_strict}. 
If $t\sim s$, then there are two possibilities: $F_s\subseteq O_t$ or $F_s\not\subseteq O_t$. If $F_s\subseteq O_t$, then $s\prec t$ by assumption, and by Fact~\ref{fact:persistence_strict}, $s\prec^\textbf{lex}_{{\vec{O}[n]}} t$.
If $F_s\not\subseteq O_t$, then there exists a $p\in set(\vec{O}[n])$ such that $t\not\in set(\vec{O}[n])$. Therefore, $s\prec^\textbf{lex}_{{\vec{O}[m]}} t$ for some $m\leq n$. By Fact~\ref{fact:persistence_strict}, $s\prec^\textbf{lex}_{{\vec{O}[n]}} t$.

 [Case 2: $t\not\sim^\textbf{lex}_{{\vec{O}[n]}}s$] Since $t\not\sim^\textbf{lex}_{{\vec{O}[n]}}s$, we have that $t\not\sim s$ by Fact~\ref{fact:pers_comparability}. It must be the case that $set(\vec{O}[n])\subseteq O_t$, otherwise, $s\prec^\textbf{lex}_{{\vec{O}[n]}} t$. By assumption, we therefore know that there exists a $v\prec t$ such that $set(\vec{O}[n])\subseteq O_v$. Since both $set(\vec{O}[n])\subseteq O_v$ and $set(\vec{O}[n])\subseteq O_t$, we have that $v\prec^\textbf{lex}_{{\vec{O}[n]}} t$. Therefore, $t\not\in min_{\preceq^\textbf{lex}_{{\vec{O}[n]}}}$. 
 % From this $v$ we can show that $t$ and $v$ are part of an infinite descending chain without a minimal element as done in Prop.~\ref{prop:general_characterisation_conditioning}. Furthermore, this chain is still present in the model after updating on $\vec{O}[n]$.

 We therefore know that for all $t\sim s$, $s\preceq^\textbf{lex}_{{\vec{O}[n]}} t$, and there is no $t\not\sim^\textbf{lex}_{{\vec{O}[n]}} s$ such that for all $v\sim^\textbf{lex}_{{\vec{O}[n]}} t$, $t\preceq ^\textbf{lex}_{{\vec{O}[n]}} v$. Therefore $min_{\preceq^\textbf{lex}_{{\vec{O}[n]}}}=\{s\}$, and thus $L^\textbf{PLAUS}_\textbf{lex}(\textbf{S}, \vec{O}[n])=\{s\}$.

 We now need to show that for all $k\geq n$, $L^\textbf{PLAUS}_\textbf{lex}(\textbf{S}, \vec{O}[k])=\{s\}$. Assume that there is some $k$ for which $L^\textbf{PLAUS}_\textbf{lex}(\textbf{S}, \vec{O}[k])\not=\{s\}$. By Fact~\ref{fact:min_preservation_sound}, we know that $s\in L^\textbf{PLAUS}_\textbf{lex}(\textbf{S}, \vec{O}[k]) $. Therefore, there must exist a $t$ such that $t\not=s$ and $t\in L^\textbf{PLAUS}_\textbf{lex}(\textbf{S}, \vec{O}[k])$. 
 
 [Case 1: $t\sim^\textbf{lex}_{{\vec{O}[n]}} s$] Since $s$ is the uniquely minimal element after the sequence $\vec{O}[n]$, it must be the case that $s\prec^\textbf{lex}_{{\vec{O}[n]}} t$. Since $\vec{O}$ is sound, by Fact~\ref{fact:persistence_strict}, $s\prec^\textbf{lex}_{{\vec{O}[k]}} t$. Contradiction.
 [Case 2: $t\not\sim^\textbf{lex}_{{\vec{O}[n]}} s$] Then $t\not\sim s$, and $F_s\subseteq set(\vec{O}[n]) \subseteq set(\vec{O}[k])\subseteq O_t$, otherwise, $t$ would have become less plausible than $s$. But then, by assumption, there exists a $v$ such that $set(\vec{O}[k])\subseteq O_v$ and $v\prec t$. Since $set(\vec{O}[k])\subseteq O_v$, and $v\prec t$, we have $v\prec^\textbf{lex}_{{\vec{O}[k]}} t$ (since they always satisfy the same propositions, their relative plausibility never changes). Thus $t\not\in min_{\preceq^\textbf{lex}_{{\vec{O}[k]}}}$, hence $t\not\in L^\textbf{PLAUS}_\textbf{lex}(\textbf{S}, \vec{O}[k])$. Contradiction.   
 \end{proof}

\end{document}